\documentclass[11pt]{article}

\usepackage{times}

\usepackage{verbatim}

\usepackage{amsthm,amsfonts,amsmath,amssymb,epsfig,color,float,graphicx,verbatim, enumitem}
\usepackage{multirow}

\newif\ifhyper\IfFileExists{hyperref.sty}{\hypertrue}{\hyperfalse}
\hypertrue
\ifhyper\usepackage{hyperref}\fi

\usepackage{enumitem}

\newcommand{\inote}[1]{\footnote{{\bf [[Ilias: {#1}\bf ]] }}}

\makeatletter
\renewcommand{\section}{\@startsection{section}{1}{0pt}{-12pt}{5pt}{\large\bf}}
\makeatother

\usepackage{framed}
\usepackage{nicefrac}

\def\nnewcolor{1}
\ifnum\nnewcolor=1
\newcommand{\nnew}[1]{{\color{red} #1}}
\fi
\ifnum\nnewcolor=0
\newcommand{\nnew}[1]{#1}
\fi

\newtheorem{theorem}{Theorem}
\newtheorem{lemma}[theorem]{Lemma}
\newtheorem{proposition}[theorem]{Proposition}
\newtheorem{corollary}[theorem]{Corollary}
\newtheorem{claim}[theorem]{Claim}
\newtheorem{fact}[theorem]{Fact}
\newtheorem{remark}[theorem]{Remark}
\theoremstyle{definition}
\newtheorem{definition}[theorem]{Definition}

\newcommand{\chowallow}{\kappa}
\newcommand{\mychows}{\vec{\chi}}
\newcommand{\chow}{\mathrm{Chow}}
\newcommand{\supp}{\mathrm{supp}}
\newcommand{\bits}{\{-1,1\}}
\newcommand{\bn}{\bits^n}
\newcommand{\R}{{\mathbb{R}}}
\newcommand{\Z}{{\mathbb Z}}
\newcommand{\N}{{\mathbb N}}
\newcommand{\Inf}{\mathrm{Inf}}
\newcommand{\dist}{\mathrm{dist}}

\newcommand{\E}{{\bf E}}

\newcommand{\U}{{\cal U}}

\newcommand{\Var}{\mathbf{Var}}
\renewcommand{\P}{\mathbf{Pr}}
\renewcommand{\Pr}{\mathbf{Pr}}
\newcommand{\pr}{\mathbf{Pr}}
\newcommand{\eps}{\epsilon}
\newcommand{\littlesum}{\mathop{{\textstyle \sum}}}

\newcommand{\la}{\langle}
\newcommand{\ra}{\rangle}

\newcommand{\sgn}{\mathrm{sign}}

\newcommand{\poly}{\mathrm{poly}}

\newcommand{\T}{\mathcal{T}}

\newcommand{\sign}{{\mathrm{sign}}}

\newcommand{\wh}[1]{{\widehat{#1}}}

\newcommand{\eqdef}{\stackrel{\textrm{def}}{=}}
\newcommand{\depth}{{\mathrm{depth}}}

\makeatletter
\let\@@span\span
\def\sp@n{\@@span\omit\advance\@multicnt\m@ne}
\makeatother

\renewcommand{\span}{\mathrm{span}}

\title{Degree-$d$ Chow Parameters Robustly Determine Degree-$d$ PTFs\\ 
(and Algorithmic Applications)}

\author{
Ilias Diakonikolas\thanks{Supported by NSF Award CCF-1652862 (CAREER) and a Sloan Research Fellowship.}\\
University of Southern California\\
{\tt diakonik@usc.edu}\\
\and
Daniel M. Kane\thanks{Supported by NSF Award CCF-1553288 (CAREER) and a Sloan Research Fellowship.}\\
University of California, San Diego\\
{\tt dakane@cs.ucsd.edu}\\
}

\begin{document}

\maketitle
%\dnote{Should the title be ``A Solution to The Degree-$d$ Chow Parameters Problem'' or ``Degree-$d$ Chow Parameters Robustly Determine PTFs'' or something?}
%\inote{Other Suggestions: ``Robust Identifiability of Degree-$d$ PTFs from their Degree-$d$ Chow Parameters and its Algorithmic Applications''}

\thispagestyle{empty}

\vspace{-0.5cm}

\begin{abstract}
The degree-$d$ Chow parameters of a Boolean function $f: \bn \to \R$ are its degree at most $d$ Fourier coefficients.
It is well-known that degree-$d$ Chow parameters uniquely characterize degree-$d$ polynomial threshold functions 
(PTFs) 
within the space of all bounded functions. In this paper, we prove a robust version of this theorem: 
For $f$ any Boolean degree-$d$ PTF and $g$ any bounded function, if the degree-$d$ Chow parameters of
$f$ are close to the degree-$d$ Chow parameters of $g$ in $\ell_2$-norm, then $f$ is close to $g$ in $\ell_1$-distance.
Notably, our bound relating the two distances is completely independent of the dimension $n$. That is,
we show that Boolean degree-$d$ PTFs are {\em robustly identifiable} from their degree-$d$ Chow parameters.
Results of this form had been shown for the $d=1$ case~\cite{OS11:chow, DeDFS14}, 
but no non-trivial bound was previously known for $d >1$.

Our robust identifiability result gives the following algorithmic applications: 
First, we show that Boolean degree-$d$ PTFs can be efficiently approximately reconstructed
from approximations to their degree-$d$ Chow parameters. This immediately implies
that degree-$d$ PTFs are efficiently learnable in the uniform distribution $d$-RFA model~\cite{BenDavidDichterman:98}.
As a byproduct of our approach, we also obtain the first low integer-weight approximations of degree-$d$ PTFs, for $d>1$.
As our second application, our robust identifiability result gives the first efficient 
algorithm, with dimension-independent error guarantees, 
for malicious learning of Boolean degree-$d$ PTFs under the uniform distribution.

The proof of our robust identifiability result involves several new technical ingredients, including
the following structural result for degree-$d$ multivariate polynomials with very poor anti-concentration:
If $p$ is a degree-$d$ polynomial where $p(x)$ is {\em very} close to $0$ on a {\em large} number of points in $\bn$,
then there exists a degree-$d$ hypersurface that exactly passes though {\em almost all} of these points. 
We leverage this structural result to show that if the degree-$d$ Chow distance between $f$ and $g$ is small, 
then we can find many degree-$d$ polynomials that vanish on their disagreement region, 
and in particular enough that forces the $\ell_1$-distance between $f$ and $g$ to also be small. 
To implement this proof strategy, we require additional technical ideas. 
In particular, in the $d=2$ case we show that for any large vector space of degree-$2$ polynomials 
with a large number of common zeroes, there exists a linear function that vanishes on almost all of these zeroes. 
The degree-$d$ degree generalization of this statement is significantly more complex, and can be viewed as an effective version 
of Hilbert's Basis Theorem for our setting.
\end{abstract}

\thispagestyle{empty}
\setcounter{page}{0}

\newpage

\section{Introduction} \label{sec:intro}

This paper is concerned with the problem of reconstructing degree-$d$ polynomial threshold functions
over the Boolean hypercube from their degree at most $d$ Fourier coefficients. Before we state
our main contributions, we provide some background and motivation for this work.

\subsection{Background} \label{sec:background}

A {\em degree-$d$ polynomial threshold function (PTF)} is a Boolean function of the form
$f(x) = \sign(p(x))$, \noindent where $p: \bits^n \to \R$ is a degree-$d$ polynomial with real
coefficients. The function $\sign(z)$ takes value $1$ for $z \geq 0$ and $-1$ otherwise.
For degree $d=1$, degree-$d$ PTFs are referred to as Linear Threshold Functions (LTFs) or Boolean
Halfspaces. Degree-$d$ PTFs are a fundamental family of Boolean functions that have been
extensively studied in a number of scientific disciplines for at least six decades~\cite{Rosenblatt:58, Chow:61, MTT:61, Dertouzos:65, MinskyPapert:68, SRK:95}.
During the past decade, there has been renewed research attention on degree-$d$ PTFs
from several perspectives of theoretical computer science, including complexity theory~\cite{Servedio:07cc,
RabaniS10, DGJ+:10, DKNfocs10, DiakonikolasS13, MZ13, Kane14GL, DeDS14-ccc, DeS14, GopalanKM15, Kane17DifDecomp} and
computational learning theory~\cite{DHK+:10, DiakonikolasOSW11, OS11:chow, FeldmanGRW12, DeDFS14, DRST14,
HKM14, Daniely15, Daniely16, DeDS17, ABL17, DiakonikolasKS18a-nasty, BhattacharyyaGS18}.
%\inote{PTFs are important and well-studied in many branches of TCS. Cite more papers carefully here.}

%\inote{How do we connect with definition below?}

In this work, we study the problem of efficiently approximating degree-$d$ PTFs over $\bn$
from their Fourier coefficients of degree at most $d$, which we will call degree-$d$ Chow parameters:

\begin{definition} \label{def:chow}
Given any function $f : \bn \to \R$, its \emph{degree-$d$ Chow Parameters}
are its Fourier coefficients of degree at most $d$, i.e., $\wh{f}(S) = \E_{x \sim \U_n}[f(x) \chi_S(x)]$,
for all $S \subseteq [n]$ with $|S| \leq d$, where $\U_n$ denotes the uniform distribution on $\bn$.
We say that the \emph{degree-$d$ Chow vector} of $f$ is
$\mychows^d_f = (\wh{f}(S))_{S \subseteq [n], |S| \leq d}$.
\end{definition}

In general, if the degree $d$ is relatively small, the degree-$d$ Chow parameters of a Boolean function $f$ provide
limited information about the function. Perhaps surprisingly, this is not the case if we know that $f$ is an LTF
or, more generally, a degree-$d$ PTF.  A classical result of C.K. Chow~\cite{Chow:61} shows
that Boolean LTFs are uniquely specified by their degree-$1$ Chow parameters.
Chow's result was later generalized by Bruck~\cite{Bruck:90} to
the class of degree-$d$ PTFs. Specifically, ~\cite{Chow:61, Bruck:90} establish the following:

\medskip

\noindent {\bf Degree-$d$ Chow's Theorem:} For any $d \geq 1$,
if $f:\bn \to \bits$ is any degree-$d$ PTF and $g: \bn \to [-1, 1]$ is any bounded function
such that $\mychows^d_f  = \mychows^d_g$, then $f(x) = g(x)$ for all $x \in \bn$.

\medskip

Unfortunately, the original proof of Chow's theorem~\cite{Chow:61}
(and its straightforward generalization to the degree-$d$ case~\cite{Bruck:90})
is non-constructive, and in particular does not suggest any algorithm
to reconstruct (even approximately) a degree-$d$ PTF from its degree-$d$ Chow parameters.
This naturally suggests the following computational problem:

\medskip

\noindent {\bf Degree-$d$ Chow Parameters Problem:}
Given (approximations to) the degree-$d$ Chow parameters of an unknown degree-$d$ PTF over $\bn$,
output an (approximate) representation of $f$ as $\sgn(p(x))$, where $p: \bn \to \R$ is a degree-$d$ polynomial.

\medskip

The (degree-$1$) Chow Parameters problem has a rich history and
has been extensively studied since the 1960s.
Specifically, researchers in various communities have been interested in finding an efficient algorithm for the problem,
including electrical engineering~\cite{Elgot:60, MTK:62, Winder:64,MTB:67,
Kaszerman:63,Winder:63,KaplanWinder:65,Dertouzos:65, Winder:69, Baugh:73, Hurst:73},
game theory and voting theory~\cite{Lapidot:72, Penrose:46,Banzhaf:65,DubeyShapley:79,EinyLehrer:89,TaylorZwicker:92,Freixas:97,
Leech:03,Carreras:04,FM:04,TT:06,APL:07, LW:98, L:02a, L:02b, dK:10, K:12, KN:12},
and computational learning theory~\cite{BDJ+:98, BenDavidDichterman:98, Goldberg:06b, Servedio:07cc, OS11:chow, DeDFS14}.
More recently, Chow's theorem and the Chow parameters have played an important role in various aspects of
complexity theory (see, e.g.,~\cite{CHIS:10, KaneW16}).
The reader is referred to~\cite{OS11:chow} for a detailed summary of previous work.

The first provably efficient algorithm for the degree-$1$ Chow parameters problem was obtained
by O'Donnell and Servedio~\cite{OS11:chow}: they gave a polynomial time algorithm that, given
sufficiently accurate approximations to the degree-$1$ Chow parameters of an unknown LTF $f$,
it outputs the weights-based representation of an LTF $h$ that is close to $f$ in (normalized) Hamming distance.
In subsequent work, De, Diakonikolas, Feldman, and Servedio~\cite{DeDFS14} gave a significantly faster
algorithm for the (degree-$1$ Chow parameters) problem.
(See Section~\ref{sec:related} for a more detailed description of prior work.)
The degree-$d$ Chow parameters problem for $d>1$ has remained open.
Prior to this work, no non-trivial upper bound was known,
even for special cases of degree-$2$ PTFs.

To facilitate the subsequent discussion, we introduce some basic notation.
For $f,g: \bn \to \R$ the {\em distance} between $f$ and $g$ is
$\dist(f, g) \eqdef \E_x[|f(x) - g(x)|]$, where the underlying distribution
will be the uniform distribution on $\bn$, unless explicitly stated otherwise.
If $\dist(f,g) \leq \eps$, we say that $f$ and $g$ are $\eps$-close.
(Note that if $f, g$ are Boolean-valued, then $\dist(f, g)  = 2 \P_x[f(x) \neq g(x)]$.)
The degree-$d$ Chow parameters also naturally induce a distance measure between
functions over the Boolean hypercube:

\begin{definition}  Let $f, g : \bn \to \R$. We define
the \emph{degree-$d$ Chow distance} between $f$ and $g$ to be
$\chow_d(f,g) \eqdef  \| \mychows^d_f  - \mychows^d_g \|_2$,
i.e., the Euclidean distance between their degree-$d$ Chow vectors.
\end{definition}

(A useful equivalent reformulation is
$\chow_d(f,g) = \sup_{p \textrm{ degree at most }d, \|p\|_2=1} \E_x[p(x)(f(x)-g(x))]$, where
the supremum is taken over all normalized multilinear polynomials of degree at most $d$.)
Using this terminology, the degree-$d$ Chow's theorem can be rephrased as follows:
If $f$ is a degree-$d$ PTF and $g$ is a bounded function such that
$\chow_d(f,g) = 0$, then $\dist(f, g) = 0.$ An immediate question that arises
when thinking about this problem is to what extent is the degree-$d$ Chow's theorem {\em robust}:
In particular, if $\chow_d(f,g)$ is small, does this necessarily imply that $f$ and $g$ are close?
Or equivalently, suppose that $\dist(f, g) = \eps>0$, where $\eps$ is a small universal constant.
Is it the case that $\chow_d(f, g)$ cannot be too small?
We note that Chow's original argument does not establish any non-trivial robustness.

For the case of LTFs ($d=1$), a sequence of works~\cite{BDJ+:98,Goldberg:06b,Servedio:07cc,OS11:chow, DeDFS14}
established robust versions of the degree-$1$ Chow's theorem with varying quantitative guarantees.
In particular, for $f, g$ with $\dist(f, g) = \eps$,~\cite{OS11:chow, DeDFS14}
showed lower bounds on $\chow_1(f, g)$ that only depend on $\eps$ and are independent of $n$.
Similarly to the algorithmic version of the problem, 
the existence of a robust version of the degree-$d$ Chow's theorem for $d>1$ 
was one of the main open questions in~\cite{DeDFS14} and has remained unresolved.
Prior to this work, no non-trivial bound was known,
even for degree-$2$ PTFs.

It turns out that the {\em robustness} question discussed above --- a purely structural question ---
is intimately related to the {\em algorithmic} question of approximately reconstructing
a degree-$d$ PTF from its degree-$d$ Chow parameters. It should be noted
that both previous works that provide efficient algorithms~\cite{OS11:chow, DeDFS14} for the
$d=1$ case establish robust versions of Chow's theorem and crucially use them for the
analysis of their algorithms. The connection between robustness and computationally
efficient reconstruction was made explicit in~\cite{DeDFS14} (see also~\cite{TTV:09short}),
where it was established that a sufficiently robust version of the degree-$d$ Chow's theorem
suffices to obtain an efficient approximation algorithm for the problem (see Theorem~\ref{thm:alg-chowd}
in Section~\ref{ssec:chow-apps} for a precise quantitative version).

%\inote{Maybe state the implication formally?}\dnote{Yes. This would be good.}
%\inote{I will just state one theorem here maybe and the rest in appendix for completeness.}

%in order to obtain an algorithmic result it suffices to establish a {\em robust} version of Chow's/Bruck's theorem.
%That is, a statement saying that if the Chow parameters are close then the function are close.
%We quantify this below.

%The previous works for LTFs establish theorems of this form, which give rise to algorithms.
%A robust version of Bruck's theorem suffices for degree-$d$ PTFs.
%This has remained an open problem.

\subsection{Our Results} \label{sec:resullts}

The main contribution of this paper is a robust version of the degree-$d$ Chow's theorem
that is completely independent of $n$. Specifically, we prove the following:

\begin{theorem}[Main Result] \label{thm:chow-d-struct}
There exists a function $\chowallow: \R \times \N \to \R$ such that the following holds:
Let $f: \bn \to \bits$ be any degree-$d$ PTF and $g: \bn \to [-1, 1]$ be an arbitrary bounded function.
If $\chow_d(f, g) \leq \chowallow(\eps, d)$, then $\dist(f, g) \leq \eps$.
\end{theorem}

Some comments are in order: The main conceptual message of Theorem~\ref{thm:chow-d-struct}
is that the function $1/\chowallow(\eps, d)$ is {\em independent} of the number of variables $n$.
Prior to our work, no structural result of this form was known with a sub-exponential
dependence on $n$, even for restricted classes of degree-$2$ PTFs and $\eps=0.49$.

We note that the growth rate of the function $1/\chowallow(\eps, d)$ established
by our current proof is very large. Specifically, $1/\chowallow(\eps, d)$ grows like
$\textrm{Ackermann}(d+O(1),1/\eps)$. We believe that the right dependence is quasi-polynomial
in $1/\eps$ for constant $d$ \footnote{A quasi-polynomial lower bound is known for $d=1$~\cite{DeDFS14}.},
 though proving such an improved bound would require additional ideas. The correct dependence on $d$ is less clear, but ought to be at least doubly exponential.

Theorem~\ref{thm:chow-d-struct} is a natural structural result on the Fourier structure
of degree-$d$ PTFs that we believe is of independent interest. Below,
we describe a number of algorithmic and structural applications of Theorem~\ref{thm:chow-d-struct}.

\paragraph{Algorithmic and Structural Applications.}
Our first algorithmic application is an efficient algorithm for the degree-$d$ Chow
parameters problem. Combined with known algorithmic machinery~\cite{TTV:09short, DeDFS14},
Theorem~\ref{thm:chow-d-struct} yields the following:

\begin{theorem}[Reconstruction of Degree-$d$ PTFs from Degree-$d$ Chow Parameters] \label{thm:chow-d-alg}
There is an algorithm that on input $\eps,\delta$, and a vector $\vec{\alpha}$
satisfying $\| \vec{\alpha} - \mychows^d_f \|_2 \leq \chowallow(\eps, d)$,
for an unknown degree-$d$ PTF $f$, has the following behavior: it runs in time
$\tilde{O}(n^{2d}) \cdot \poly(1/\chowallow(\eps, d)) \cdot \log(1/\delta)$ and
outputs a vector $(H_S)_{S \subseteq [n], |S| \leq d}$,
such that with probability at least $1 - \delta$, the degree-$d$ PTF
$h(x)=\sign(\sum_S H_S \prod_{i \in S} x_i)$ satisfies $\Pr_x [f(x) \neq h(x)] \leq \eps.$
\end{theorem}

In words, we obtain an algorithm for the degree-$d$ Chow parameters problem that,
for any constant accuracy $\eps$, runs in time $\tilde{O}_d(n^{2d})$.
As an immediate corollary of Theorem~\ref{thm:chow-d-alg}, we obtain an algorithm
with similar running time for learning degree-$d$ PTFs in the uniform distribution
$d$-RFA model of Ben-David and Dichterman~\cite{BenDavidDichterman:98}.
In this learning model, the learner can only observe a desired subset of coordinates
of each unlabeled example of size at most $d$. See Section~\ref{ssec:chow-apps}
for a detailed statement.

The algorithm of Theorem~\ref{thm:chow-d-alg} can be shown to output a degree-$d$ PTF
with integer weights whose sum of squares is at most $n^d \cdot \poly(1/ \chowallow(\eps, d))$.
Hence, we obtain the first non-trivial bounds on approximating arbitrary degree-$d$ PTFs
using degree-$d$ PTFs with small integer weights.

\begin{theorem} [Low Integer-Weight Approximation for Degree-$d$ PTFs] \label{thm:lowwt}
Let $f: \bn \to \bits$ be a degree-$d$ PTF.  There is a degree-$d$ PTF
$h(x)=\sign(\sum_{S \subseteq [n], |S| \leq d} H_S \prod_{i \in S} x_i)$
such that $\Pr_x [f(x) \neq h(x)] \leq \eps$ and the weights $H_S$ are integers
that satisfy
$\littlesum_{S} H_S^2 = O(n^d) \cdot \poly(1/ \chowallow(\eps, d))$.
\end{theorem}

A number of previous works~\cite{Servedio:07cc, DiakonikolasS13, DeDFS14}
obtained low integer-weight approximators to LTFs, culminating in the
near-optimal\footnote{A construction of~\cite{Hastad:94} implies a lower bound of $\max\{n^{1/2}, (1/\eps)^{\Omega(\log \log(1/\eps))}\}$.}
integer weight bound of $O(n) \cdot (1/\eps)^{O(\log^2(1/\eps))}$~\cite{DeDFS14}.
%We note that ideas developed in work on low-weight approximators for LTFs
%have been useful in other contexts, including hardness of approximation~\cite{FeldmanGRW12} and
%property testing~\cite{MORS:10}.
For $d>1$, no non-trivial bound was known prior to our work.
We note that~\cite{DSTW14} gave
low integer-weight approximators for degree-$d$ PTFs, but
the degree of the approximating PTF is $(1/\eps)^{\Omega(d)}$, as opposed to $d$.

Our main structural result also has algorithmic implications for the problem of learning Boolean
degree-$d$ PTFs in the malicious learning model of Valiant, Kearns and Li~\cite{Valiant:85, KearnsLi:93}.
The {\em malicious noise model} is a generalization of the PAC model in which an adversary can
arbitrarily corrupt a small constant fraction of both the unlabeled data points and their labels.
Using the machinery of~\cite{DiakonikolasKS18a-nasty}, we obtain an algorithm
that learns Boolean degree-$d$ PTFs in the presence of a small
constant fraction of corrupted data:

\begin{theorem}[Learning Boolean Low-Degree PTFs with Nasty Noise] \label{thm:malicious}
There is a polynomial-time algorithm for learning Boolean degree-$d$ PTFs in the presence of malicious noise
with respect to the uniform distribution on $\bn$. Specifically, if $\chowallow(\eps,d)$ is the noise rate, the algorithm runs
in $\poly(n^d, 1/\eps)$ time and outputs a hypothesis degree-$d$ PTF $h(x)$ that with high probability satisfies
$\Pr_{x} [h(x) \neq f(x)] \leq  \eps$, where $f$ is the unknown target PTF.
\end{theorem}

We note that the algorithm establishing Theorem~\ref{thm:malicious} was given in~\cite{DiakonikolasKS18a-nasty}.
Our Theorem~\ref{thm:chow-d-struct} is the missing technical ingredient to prove correctness of this algorithm
for the setting of the uniform distribution on the hypercube. See Section~\ref{ssec:app-malicious} for a detailed
explanation.

\subsection{Related and Prior Work} \label{sec:related}
In this section, we review some relevant prior work on the
degree-$1$ version of the Chow parameters
problem~\cite{Goldberg:06b, Servedio:07cc, OS11:chow, DeDFS14}.
Goldberg~\cite{Goldberg:06b} showed that for $f$ an $n$-variable LTF and $g$ any Boolean
function, if $\dist(f,g) = \eps$ then $\chow_1(f,g) \geq (\eps/n)^{O(\log(n/\eps)\log(1/\eps))}$.
In the same setting, \cite{Servedio:07cc} obtained a lower bound of $\chow_1(f,g) \geq 1/(\poly(n) \cdot 2^{\tilde{O}(1/\eps^2)})$
and  \cite{OS11:chow} obtained the bound $\chow_1(f,g) \geq 2^{-\tilde{O}(1/\eps^2)}$.
Finally,~\cite{DeDFS14} improved the latter lower bound to $\chow_1(f,g) \geq \eps^{O(\log^2(1/ \eps))}$,
which is the best known bound to date and qualitatively nearly-matches an upper bound of $\eps^{O(\log \log (1/\eps))}$.
Building on their structural result, \cite{OS11:chow} gave an algorithm for the degree-$1$ Chow parameters
problem that finds an $\eps$-approximator to the unknown LTF in time $\poly(n) \cdot 2^{2^{\tilde{O}(1/\eps^2)}}$.
\cite{DeDFS14} gave a new algorithm for the problem that, combined with their structural result, was shown
to run in time $\poly(n) \cdot (1/\eps)^{O(\log^2(1/\eps))}$. The algorithm of \cite{DeDFS14} straightforwardly
generalizes to the degree-$d$ case, but its analysis hinges on a robust version of the degree-$d$ Chow's theorem,
which we prove in this work.

\subsection{Our Techniques} \label{sec:techniques}
In this section, we provide an overview of our techniques that
lead to the proof of Theorem~\ref{thm:chow-d-struct} in tandem
with a comparison to prior work.
We start by reviewing previous approaches that give
robust versions of the degree-$1$ Chow's theorem.
Let $f$ be an LTF and $g$ a Boolean function such that
$\dist(f, g) = \eps$, $\chow_1(f, g)  = \delta$;
we would like to show that $\delta$ cannot be too small (as a function of $\eps$ and,
potentially, the dimension $n$). When one tries to robustify the
original proof of Chow~\cite{Chow:61}, one finds that
the argument goes through unless the LTF
$f(x)=\sgn(L(x))$ has $|L(x)|$ very close to $0$
on almost all the points where $f$ and $g$ differ.
In other words, if $L(x)  = w \cdot x - \theta$
is {\em anti-concentrated} around the origin,
i.e., the fraction of points $x \in \bn$ such that $|L(x)|$
is very close to $0$ is small, then Chow's argument can be naturally extended.
Unfortunately, this is not always the case: it is quite possible that $|L(x)|$ is
very close to $0$ for a significant fraction of points $x \in \bn$, which makes it
challenging to robustify Chow's argument.

Two approaches have been proposed to circumvent the above obstacle.
The idea in~\cite{OS11:chow} (implicit in~\cite{Servedio:07cc})
is to approximate an arbitrary LTF by an LTF with ``good anti-concentration''.
For this idea to work, it is crucial that the normalized Hamming distance
between $f$ (the original LTF) and its approximator $f'$ to be very small
compared to the anti-concentration radius. While such an approach was shown
to be feasible for the degree-$1$ case, we do not know if it is possible to extend
even to the case of degree-$2$ PTFs.

On the other hand, Goldberg~\cite{Goldberg:06b} and~\cite{DeDFS14}
(that builds on and substantially strengthens~\cite{Goldberg:06b}) uses a more
direct geometric view of the problem. Roughly speaking,
it is shown in~\cite{Goldberg:06b, DeDFS14} that if
$L(x)$ has ``very poor'' anti-concentration, then
the linear function $L(x)$ satisfies certain important structural properties.
More specifically, suppose that
for some moderately large $\eps$ and very small $\delta$
that $\Pr_x[|L(x)|<\delta] > \eps$. Then there exists a
linear polynomial $L'(x)$ so that all but a tiny fraction
of the points $x$ in the disagreement region between $f$ and $g$
satisfy $L'(x)=0$ --- as opposed to $L(x)$, which is very close to $0$.
By slightly modifying $g$, we can reduce to the case where
\emph{all} of the discrepancies lie on the hyperplane $L'(x)=0$.
This allows us to renormalize $L$, by taking it modulo $L'$,
and potentially find a second linear function
on which nearly all of the discrepancies lie.
Repeating this process, we can eventually find a large number of linear functions
so that nearly all of the disagreements between $f$ and $g$ lie
on the intersection of the corresponding hyperplanes. However, given enough such functions,
there will no longer be enough points for this to be the case, yielding a contradiction.

At a high-level, our approach for the degree-$d$ case is a generalization of
~\cite{Goldberg:06b, DeDFS14}. Firstly, we note that the robustification of
the degree-$1$ Chow's result still works for the degree-$d$ case,
unless the degree-$d$ polynomial $p$ defining our degree-$d$ PTF $f(x) = \sgn(p(x))$
has the same kind of very poor anti-concentration as before (see Claim~\ref{claim:small-p-in-diff}).
We will next need to show that this implies that almost all of the disagreements lie on a degre-$d$
polynomial hypersurface (Proposition~\ref{prop:degd-poly-disagreement-region}).
While~\cite{Goldberg:06b, DeDFS14} accomplish this for $d=1$ by a careful analysis of the vectors
perpendicular to the discrepancy points (in order to get nearly optimal quantitative bounds),
our techniques for the degree-$d$ case are less accurate. We first need some sort of
general anti-concentration result --- which we obtain via a combination of a regularity lemma~\cite{DSTW14}
and an invariance principle~\cite{MOO10} ---
to show that this kind of bad anti-concentration implies that
there exists a small set of coordinates $S \subseteq [n]$,
so that upon fixing the variables in $S$, most of the disagreements between $f, g$
reduce to polynomials in the remaining coordinates with small $\ell_2$-norm
(see Corollary~\ref{cor:small-p-xs}).
We then show that by scaling $p$, we can get a polynomial that is in some sense ``nearly integral'',
and establish that the integer part must vanish on almost all of the points where $|p(x)|$
is extremely close to $0$ (see Fact~\ref{fact:mult-int} and the paragraph preceding it
for a more detailed overview).

It turns out that the much more challenging part of our proof
is to generalize the iteration of the above result.
We have established that if $p$ has very bad anti-concentration, then
almost all of its near zeroes lie on the zero-set of a degree-$d$ polynomial.
By iterating this fact, it is not hard to show that if $f$ and $g$ have very small degree-$d$
Chow distance, we can find a whole sequence $q_1,q_2,\ldots$ of linearly independent,
degree-$d$ multilinear polynomials so that almost all disagreements of $f$ and $g$
lie on the joint zero set of the $q_i$'s. When $d=1$, completing the proof from this point would
be easy, since at most a $2^{-m}$-fraction of points in $\bn$ can lie on the joint zeroes
of $m$ linear polynomials (see Fact~\ref{fact:affine-subspace}).

Unfortunately, even for degree-$2$ polynomials, this statement is false.
For example, consider the sequence of polynomials $q_i = (x_0+1)(x_i+1)$.
These are linearly independent, however half of all points (those with $x_0=-1$) are joint zeroes.
This can happen because almost all of our zeroes lie on a hyperplane.
In fact, for the degree-$2$ case, we can show (see Proposition~\ref{prop:linear-form-zero-set-deg2-subspace})
that this is essentially the only thing that can go wrong. In particular, we prove
that given sufficiently many linearly independent degree-$2$ polynomials,
we can replace them with a {\em single} degree-$1$ polynomial without loosing too many disagreements
in the zero set. Iterating this result, until we have enough linear polynomials,
yields a contradiction as before.

The higher degree case runs along the same lines, however the recursion becomes somewhat more complicated.
At each stage of the process, we maintain an {\em ideal} $I$ of polynomials so that almost all of the disagreements
of $f$ and $g$ lie on zeroes of $I$. We show that if the original degree-$d$ Chow distance was small enough,
we can always add another degree-$d$ polynomial to $I$. From here what we need can be seen
as a robustification of Hilbert's Basis Theorem for our setting.
The Hilbert Basis Theorem~\cite{Hilbert1890} (see, e.g.,~\cite{Cox:2007})
says that, starting with an ideal, if one repeatedly adds new polynomials,
this process must eventually terminate (perhaps with $I$ being the unit ideal).
Unfortunately, the number of rounds of this iterative process is unbounded,
and even under reasonable restrictions, will still depend on the number of variables $n$.
What we establish here is that (1) If the added polynomials are all degree at most $d$, and
(2) If we are allowed to throw sets of negligible mass out of the associated variety,
then we can actually obtain an upper bound on the number of steps. To achieve this,
we show that we can replace sufficiently many degree-$k$ polynomials
by a {\em single} degree-$(k-1)$ polynomial without losing too much probability mass.
The basic idea of the proof is to use a degree-$d$ version of the Littlewood-Offord lemma~\cite{MekaNV16}
to show that if there are many degree-$k$ polynomials with a large number of joint zeroes,
then either (1) there is a degree-$(k-1)$ polynomial which vanishes on almost all of them,
or (2) there is a small set of coordinates so that all of the polynomials depend
on only these coordinates (and thus the dimension of the space they span is bounded).
The above ingredients suffice in order to obtain a contradiction in the degree-$d$ case.

\subsection{Organization} \label{ssec:org}

The structure of this paper is as follows: In Section~\ref{sec:prelims}, we introduce the
mathematical background required for our results. Section~\ref{sec:struc} contains the proof of our
main structural result (Theorem~\ref{thm:chow-d-struct}). In Section~\ref{sec:apps}, we present
our algorithmic and structural applications. In Section~\ref{sec:conc}, we conclude
with a few open problems.

%\newpage

\section{Preliminaries} \label{sec:prelims}

\paragraph{Notation.}
We start by establishing basic notation.  For $n \in \Z_+$, we write $[n]$ to denote $\{ 1, 2, \ldots, n\}$.
We write $\E[X]$ and $\Var[X]$ to denote expectation and variance of a random variable $X$, where the
underlying distribution will be the uniform distribution $\U_n$ on $\bn$, unless explicitly stated otherwise.
For $x \in \bn$ and $S \subseteq [n]$ we write $x_S$ to denote
$(x_i)_{i\in S}$. For a function $f : \bits^n \to \R$ and $q \geq 1$, we
denote by $\|f\|_q$ its $l_q$-norm, i.e., $\|f\|_q \eqdef \E_{x}[|p(x)|^q]^{1/q}$.
For Boolean functions $f,g: \bn \to \R$ the
distance between $f$ and $g$, denoted $\dist(f,g)$, is defined by
$\dist(f, g) \eqdef \E_x[|f(x) - g(x)|]$. If $\dist(f,g) \leq \eps$, we say that $f$ and $g$ are $\eps$-close.
Note that if $f, g$ are Boolean-valued, i.e., take values in $ \{ \pm 1 \}$, then
$\dist(f, g)  = 2 \P_x[f(x) \neq g(x)]$. The {\em disagreement region} between $f$ and $g$
is defined by $D(f, g)  \eqdef \{ x \in \bn: f(x) \neq g(x)\}$.
For a multilinear polynomial $P: \R^n \to \R$ with $P(x) = \sum_{S \subseteq [n]} P_S \prod_{i \in S} x_i$,
we denote by $\supp(P) = \{S \subseteq [n]:  P_S \neq 0 \}$
and we call $|\supp(P)|$ the support size of $P$.

\paragraph{Fourier Analysis and Influences.}
We consider functions $f : \bn \to \R$ and we think of the inputs $x$
as being distributed according to the uniform distribution $\U_n$.
The set of such functions forms a $2^n$-dimensional inner product space with inner product given by
$\la f, g \ra = \E_x[f(x)g(x)]$. The set of functions $(\chi_S)_{S \subseteq [n]}$ defined by $\chi_S(x) = \prod_{i \in S} x_i$
forms a complete orthonormal basis for this space. Given a function $f : \bn \to \R$ we define its
\emph{Fourier coefficients} by $\wh{f}(S) \eqdef \E_x[f(x) \chi_S(x)]$,
and we have that $f(x) = \sum_{S \subseteq [n]} \wh{f}(S) \chi_S(x)$.
%We refer to the maximum $|S|$ over all nonzero $\wh{f}(S)$ as the \emph{Fourier degree} of $f$.
As a consequence of orthonormality, we have \emph{Plancherel's identity} $\la f, g \ra = \sum_{S \subseteq [n]} \wh{f}(S) \wh{g}(S)$,
which has as a special case \emph{Parseval's identity}, $\E_x[f(x)^2] = \sum_{S\subseteq[n]} \wh{f}(S)^2$.
From this it follows that for every $f : \bn \to \bits$ we have $\sum_S \wh{f}(S)^2 = 1$.
The expectation and the variance of $f : \bn \to \R$ can be expressed in terms of the Fourier
coefficients of $f$ by $\E[f] = \E_x[f(x)] = \widehat{f}(\emptyset)$ and
$\Var[f] = \Var_x [f(x)] = \littlesum_{\emptyset \neq S \subseteq [n]} \widehat{f}(S)^2.$
The \emph{influence} of variable $i$ on $f : \bn \to \R$ is $\Inf_i(f) \eqdef \sum_{S \ni i} \widehat{f}(S)^2$ and the
\emph{total influence} of $f$ is $\Inf(f) = \sum_{i=1}^n \Inf_i(f) = \sum_{S \subseteq [n]} |S| \wh{f}(S)^2$.

%\inote{What are the facts that we need below? I think the Chernoff bound is one maybe and the corollary of anti-concentration
%for regular degree-$d$ polys. And that's it. Add careful references to all these facts.}

\paragraph{Useful Probability Bounds.}
We will need the following well-known concentration bound for degree-$d$ polynomials, a
simple corollary of hypercontractivity (see, e.g., Theorem~9.23 in~\cite{ODonnell:ABF}):

\begin{fact}\label{thm:deg-d-chernoff}
Let $p:\bits^n \to \R$ be a degree-$d$ multilinear polynomial. For any $t>e^d$,
we have that $\P_x[|p(x)|\geq t \|p\|_2]\leq \exp(-\Omega(t^{2/d}))$.
\end{fact}

We say that a polynomial $p: \bits^n \to \R$ is $\tau$-regular
if $\max_{i \in [n]} \Inf_{i}(p) \leq \tau \cdot  \Inf(p)$.
Our second technical fact is that regular polynomials over the hypercube
are anti-concentrated. This follows  by combining the invariance principle~\cite{MOO10} and
Gaussian anti-concentration~\cite{CW:01} (see Claim~4.2 in~\cite{DSTW14}
for an explicit reference):

\begin{claim} \label{claim:reg-ac}
Let $p:\bn \to \R$ be a $\tau$-regular degree-$d$ multilinear polynomial.
Then it holds that $\P_{x}[|p(x)| \leq \tau \|p\|_2] \leq O(d \tau^{1/(8d)}).$
\end{claim}

Our proof makes essential use of a degree-$d$ version~\cite{MekaNV16} of the classical
Littlewood-Offord lemma~\cite{LO:43, Erd:45}. To state it, we need the following definition:

\begin{definition}[\cite{RazborovV13, MekaNV16}][Rank of Multilinear Polynomials]
For a degree-$d$ multilinear polynomial on $n$ variables $P(x) = \sum_{S \subseteq [n], |S| \leq d} P_S \cdot \chi_S(x)$,
the {\em rank} of $P$, denoted by $\mathrm{rank}(P)$, is the largest integer $r$ such that
there exist disjoint sets $S_1, \ldots, S_r \subseteq [n]$ of size $d$ with $P_{S_j} \neq 0$
for $j \in [r]$.
\end{definition}

\begin{theorem}[\cite{MekaNV16}, Theorem~1.6] \label{thm:LO}
For any degree-$d$ multilinear polynomial $P$ in $n$ variables of rank $r \geq 2$ and any point $t \in \R$
we have that $\P_x [P(x)=t] \leq \frac{ d^{O(d^2)} \cdot (\log r)^{O(d \log d)}}{\sqrt{r}}$.
\end{theorem}

We note that the above is not quite the theorem from \cite{MekaNV16}.
They show that if there are $r$ disjoints sets of coordinates $S_i$ so that the $|P_{S_j}|$
are all at least $1$, then the probability that $|P(x)-t|<1$ is bounded by $\frac{ d^{O(d^2)} \cdot (\log r)^{O(d \log d)}}{\sqrt{r}}$.
Our result follows easily from this by replacing $P$ by $m \cdot P$ for $m$ an integer larger than
any of the $1/|P_{S_j}|$'s. Applying \cite{MekaNV16}'s original result to bound the probability
that $|m \cdot P(x)-m t|<1$ gives Theorem~\ref{thm:LO}.
We will make use of Theorem~\ref{thm:LO} in the form
of the following corollary:
\begin{corollary}\label{LOCor}
For any degree-$d$ multilinear polynomial $P$, if $\P_x[P(x)=t] > \eta$,
then there exists a set $S$ of at most $d^{O(d^2)}\eta^{-3}$ coordinates so that
every non-zero degree-$d$ term in $P$ has at least one variable of $S$ in it.
\end{corollary}
\begin{proof}
By the contrapositive of Theorem \ref{thm:LO}, the rank of $P$ must be at most $d^{O(d^2)}\eta^{-3}$.
Let $S_1,S_2,\ldots,S_r$ be a maximal set of disjoint subsets of $[n]$ with $|S_i|=d$
and $P_{S_i}\neq 0$. We claim that $S=\bigcup S_i$ suffices. This is because any other
$S'\subset [n]$ of size $d$ with $P_{S'}\neq 0$ must intersect some $S_i$,
and thus intersect $S$. Furthermore, we have that $|S|=d \cdot r \leq d^{O(d^2)}\eta^{-3}$.
\end{proof}

%\newpage

\section{Main Structural Result: Proof of Theorem~\ref{thm:chow-d-struct}} \label{sec:struc}

In this section, we prove Theorem~\ref{thm:chow-d-struct}. 
Let $f$ be any Boolean degree-$d$ PTF and $g: \bn \to [-1, 1]$ be any bounded function.
We will show that if $\chow_d(f, g)$ is sufficiently small, then $\dist(f, g)$ is small.

The structure of this section is as follows: In Section \ref{ssec:g-boolean},
we reduce to the case that $g$ is {\em Boolean-valued}. In Section \ref{ssec:degd-poly-vanishes}, 
we develop our first major technical tool. In particular, we prove a generalization of a 
structural result in~\cite{Goldberg:06b, DeDFS14}, showing that if $f$ and $g$ have abnormally 
small degree-$d$ Chow distance, then there must be some degree-$d$ polynomial that exactly vanishes 
on almost all of their points of disagreement (Proposition \ref{prop:degd-poly-disagreement-region}). 
We then proceed to apply Proposition \ref{prop:degd-poly-disagreement-region} in order
to obtain a contradiction for the assumption that $f$ and $g$ have degree-$d$ Chow distance too small 
relative to their $\ell_1$-distance. In Section \ref{ssec:deg2}, we start by showing this in the degree-$2$ case. 
In Section \ref{ssec:degd}, we generalize to higher degrees.

\subsection{Reduction to the Case that $g$ is Boolean-valued} \label{ssec:g-boolean}

We begin by showing that it suffices to prove Theorem~\ref{thm:chow-d-struct}
for the special case that the function $g$ is Boolean-valued, as opposed to bounded. 

The idea of the proof is fairly simple: Let $g$ be a $[-1,1]$-valued function. 
We can randomly round $g$ to a Boolean-valued function 
in such a way as to maintain its expected degree-$d$ Chow parameters 
and distance from $f$. Furthermore, it is not hard to show that if the dimension $n$ is large 
(which can be achieved by introducing new irrelevant variables if necessary), 
the errors are, on average, quite small. 
That is, for any $f, g$ we can find a Boolean-valued function $g_0$ 
with $\dist(f,g) \approx \dist(f,g_0)$ and $\chow_d(f,g)\approx \chow_d(f,g_0)$. 
Therefore, if we have a theorem relating distance to degree-$d$ 
Chow distance for $f$ and $g_0$, essentially the same statement applies to $f$ and $g$.
More formally, we have:

\begin{lemma} \label{lem:reduction}
Suppose that for some $\eta,\delta,d>0$ that for any $n$ and any degree-$d$ PTF $f$ and Boolean function $g$ in $n$ variables with $\dist(f,g) > \eta$ have $\chow_d(f,g) \geq \delta$. Then for those same $\eta,d,\delta$, $f$ any degree-$d$ PTF in any number of variables and $g$ any function valued in $[-1,1]$ with $\dist(f,g) > \eta$, we must have $\chow_d(f,g)\geq \delta$.
\end{lemma}
\begin{proof}
Assume for sake of contradiction that we have a degree-$d$ PTF, $f$, and a bounded function 
$g:\bn \to [-1,1]$ with $\dist(f,g) >  \eta$ and $\chow_d(f,g) < \delta$. By introducing irrelevant variables if necessary, 
we can make the dimension $n$ as large as we like. Define a random rounding $g_0$ of $g$ as follows: 
$g_0(x) = 1$ with probability $(g(x)+1)/2$ and $g_0(x)=-1$ otherwise. Furthermore, the $g_0(x)$'s are independent of each other. 
Notice that $\dist(f,g) = \E[\dist(f,g_0)]$. Furthermore, since $\dist(f,g_0)$ is the average of $2^n$ independent and 
bounded random variables, we have that $\dist(f,g_0)=\dist(f,g)+O(2^{-n/3})$ with high probability. 
Similarly, for any subset $S$ of coordinates we have that $\E_{g_0}[\E_x[g_0(x)\chi_S(x)] ] = \E_x[g(x)\chi_S(x)]$, 
where again it is an average of $2^n$ independent, bounded random variables. Thus, with high probability over the choice of $g_0$, 
every degree at most $d$ Fourier coefficient of $g_0$ is within $2^{-n/3}$ of the corresponding parameter of $g$. 
Thus, with high probability over the choice of $g_0$ we have that 
$\dist(f,g_0) = \dist(f,g)+O(2^{-n/3})$ and $\chow_d(f,g_0) = \chow_d(f,g)+O(n^d 2^{-n/3})$. 
For $n$ sufficiently large, this implies that we have a Boolean $g_0$ so that $\dist(f,g_0) > \eta$ and $\chow_d(f,g_0) < \delta$, 
which is a contradiction.
\end{proof}

Lemma~\ref{lem:reduction} implies that, in order to establish Theorem~\ref{thm:chow-d-struct},
it suffices to prove the following:

\begin{theorem} \label{thm:chow-d-struct-boolean}
There exists a function $\chowallow: \R\times \N \to \R$ such that the following holds:
Let $f: \bn \to \bits$ be any degree-$d$ PTF and $g: \bn \to \bits$ be an arbitrary Boolean-valued function.
If $\chow_d(f, g) \leq \chowallow(\eps, d)$, then $\dist(f, g) \leq \eps$.
\end{theorem}

In the main part of this section, we prove Theorem~\ref{thm:chow-d-struct-boolean}.

\subsection{Existence of Degree-$d$ Polynomial that Vanishes in the Disagreement Region} \label{ssec:degd-poly-vanishes}

Our first result shows that there exists a polynomial that
captures almost all the disagreement region, in the sense
that it vanishes on most such points.
Formally, we have:

\begin{proposition} \label{prop:degd-poly-disagreement-region}
Fix $0<\delta < \eta < 1$ such that $\eta = 1/\poly_d(\log\log(1/\delta))$ is sufficiently large.
Let $f(x) = \sgn(p(x))$ be a degree-$d$ PTF and $g: \bn \to \bits$ be
a Boolean function such that $\chow_d(f, g) \leq \delta$.
Then there exists a degree-$d$ multilinear polynomial
$r: \R^n \to \R$ such that the following holds:
\begin{itemize}
\item[(i)] $\Pr_x \left[(f(x) \neq g(x)) \cap (r(x) \neq 0)\right] < \eta$.
\item[(ii)] $\E_x [p(x) r(x)] \neq 0$, i.e., $r$ non-trivially correlates with $p$.
\end{itemize}
\end{proposition}
Note that although $\delta$ here will be much much smaller than $\eta$, our application of this lemma will only rely on the fact that $\eta$ goes to $0$ with $\delta$ for fixed $d$.
\begin{proof}
Our proof proceeds in several stages. We begin by noting that, by a generalization of Chow's original argument,
if the degree-$d$ Chow distance between $f$ and $g$ is small, then nearly all of the discrepancies 
between $f = \sgn(p(x))$ and $g$ have $|p(x)|$ very small (see Claim \ref{claim:small-p-in-diff}). 
The intuition is that if the polynomial $p$ is regular, then this cannot happen for very many points (by Claim \ref{claim:reg-ac}). 
In order to reduce to this regular case, we apply the regularity lemma of \cite{DSTW14}. 
This will provide us with a small set $S$ of coordinates (independent of $n$) so that 
for most of our points of disagreement, after fixing the coordinates in $S$, our polynomial 
must be reduced to one with small $\ell_2$ norm. From here we use basic techniques in Diophantine 
approximation theory to approximate a multiple of $p$ by a polynomial $r$ that is in some sense integral. 
This allows us to show that on all but a small number of disagreements, not only does the restriction of $r$ 
(obtained by fixing the coordinates in $S$) have small $\ell_2$ norm, but the restriction must in fact be $0$. 
This $r$ will thus satisfy the necessary requirements of our proposition.

We now proceed with the detailed proof.
Suppose that $f(x) = \sgn(p(x))$, where $p: \R^n \to \R$ is a degree-$d$ multilinear polynomial
that without loss of generality satisfies $\|p\|_2=1$. Let $D = D(f, g)$ be the disagreement
region between $f$ and $g$. If $\Pr_x[D] < \eta$, the proposition follows trivially
by taking $r \equiv p$. We will hence assume that $\Pr_x[D] \geq \eta$.

We show that the upper bound on the degree-$d$ Chow distance 
implies that all but a few elements of $D$ have $|p(x)|$ very small. 
More concretely, we have the following claim:

\begin{claim} \label{claim:small-p-in-diff}
Let $D' = \{x \in D: |p(x)| > \delta/\eta  \}$. Then, $\Pr_{x}[D'] \leq \eta/2$.
\end{claim}
\begin{proof}
The proof follows by a simple argument very similar to the original proof
of Chow's theorem. Let $p(x) = \sum_{S \subseteq [n], |S| \leq d} p_S \chi_S(x)$
and note that $\|p\|_2^2 = \sum_{S \subseteq [n], |S| \leq d} p_S^2 =1$.
We can write
\begin{equation} \label{eqn:chow-argument}
\E_x[|(f(x)-g(x)) \cdot p(x)|] = \E_x[(f(x)-g(x)) \cdot p(x)]  =
\sum_{S \subseteq [n], |S| \leq d} p_S \cdot (\wh{f}(S)- \wh{g}(S))
\leq  \|p\|_2 \cdot \chow_d(f, g) \leq \delta \;,
\end{equation}
where the first equality uses the fact that $(f(x)-g(x))  \cdot p(x)$ is non-negative for all $x \in \bn$, the second equality
is Plancherel's identity, and the inequality is Cauchy-Schwarz. Then, we can write
\begin{equation} \label{eqn:large-p-in-D}
\E_x[|(f(x)-g(x)) \cdot p(x)|] \geq \Pr_{x}[D'] \cdot  2 \cdot \min_{x \in D'}|p(x)| > \Pr_{x}[D'] \cdot (2\delta/\eta)\;.
\end{equation}
Combining \eqref{eqn:chow-argument} and \eqref{eqn:large-p-in-D},
we get that $\Pr_{x}[D'] \leq \eta/2$,
completing the proof of Claim~\ref{claim:small-p-in-diff}.
\end{proof}

For $S \subseteq [n]$, we can partition the coordinates as $x = (x_S, x')$,
where $x'  = x_{[n]\setminus S}$
and rewrite the degree-$d$ multilinear polynomial $p(x)$ as
$p_{x_S}(x') = p(x_S, x')$. For any fixed assignment to $x_S$, we will
view $p_{x_S}(x')$ as a degree-$d$ multilinear polynomial in $x'$.

We will require the following structural lemma showing that
there exists a set $S \subseteq [n]$ of coordinates, whose size is independent of the dimension $n$,
such that for at least $1-\eta/2$ fraction of points $x \in \bn$,
the restricted polynomial $p_{x_S}$ has $\ell_2$-norm not much larger than $|p(x)|$:

\begin{lemma} \label{lem:small-norm-restr}
There exists $S \subset [n]$ of size $|S| \leq 2^{\poly_d(1/\eta)}$ such that
$$\Pr_x[\|p_{x_S}\|_2 > \poly_d(1/\eta) \cdot |p(x)|] < \eta/2.$$
\end{lemma}

\begin{proof}
We start by applying a regularity lemma to the polynomial $p$.
We will use the following statement:

\begin{theorem}[\cite{DSTW14}] \label{thm:regularity}
Let $f(x) = \sgn(p(x))$, where $p: \bn \to \R$ is a degree-$d$ multilinear polynomial and $0< \tau < 1.$
Then $f$ is equivalent to a decision tree $\T$, of depth
$\depth(d,\tau) := {\frac 1 \tau} \cdot \big(d \log {\frac 1 \tau}\big)^{O(d)}$
with variables at the internal nodes and a degree-$d$
PTF $f_\rho = \sign(p_\rho)$ at each leaf $\rho$,  with the
following property: with probability at least $1 - \tau$,
a random path\footnote{A random path corresponds to the standard
uniform random walk on the tree.} from the root reaches a leaf $\rho$
such that either (i) $p_\rho$ is $\tau$-regular, or (ii) $\Var[p_{\rho}] < \tau \cdot \E[p_\rho^2]$
\end{theorem}

\begin{remark}
{\em The above statement is not explicit in~\cite{DSTW14},
but easily follows from their proof. Specifically, Theorem~\ref{thm:regularity} follows
mutatis mutandis by taking the parameter $\beta$ in their Lemma~3.5 to satisfy
$1/\tau = \Theta (\log(1/\beta)^d)$ and noting that the restriction $\rho$ in its statement
satisfies their Definition~3.6.}
\end{remark}

We call a leaf {\em Good} if it satisfies (i) or (ii) in the statement of Theorem~\ref{thm:regularity}
above. By an application of Theorem~\ref{thm:regularity} for the polynomial $p$
defining our degree-$d$ PTF $f$
and $\tau \eqdef \Theta (\eta/d)^{8d}$ (with the implied constant sufficiently small),
we obtain a decision tree of depth  $\depth(d,\tau) = (d/\eta)^{O(d)}$
such that with probability at least $1-\tau$
a random path in the tree leads to a good leaf,
i.e., $\Pr_x[\rho \textrm{ is Good}] \geq 1-\tau.$
We show the following claim:

\begin{claim} \label{claim:p-rho}
We have that
$\Pr_x \left[ |p(x)| > \poly_d(\eta) \cdot \|p_{\rho}\|_2 \right] > 1-\eta/4$.
\end{claim}
\begin{proof}
Note that the restriction $\rho$ defining the set of variables
that are fixed in the path from the root to the corresponding leaf of the tree depends on the input $x$.
By Theorem~\ref{thm:regularity},  $\Pr_x[\rho \textrm{ is Good}] \geq 1-\tau.$
We condition on this event and analyze each case separately.

Consider a restriction $\rho$ satisfying (i). In this case, since $p_{\rho}$ is $\tau$-regular
the polynomial $p_{\rho}(x')$ is anti-concentrated. Specifically, Claim~\ref{claim:reg-ac}
gives that $\Pr_{x'} [|p_{\rho}(x')| \leq \tau \cdot \|p_{\rho}\|_2 ] \leq \eta/8$.

Consider a restriction $\rho$ satisfying (ii). In this case, the constant term of $p_{\rho}$
is very large and it is very unlikely that the non-constant term dominates.
Let $p_{\rho}(x') = p'(\rho)+q(\rho, x')$, where $x'$ are the variables not fixed
by $\rho$. Then $p'(\rho)$ is the constant term and we note that
$\Var[p_{\rho}] = \|q\|_2^2$ and $\E[p_{\rho}^2] = (p'(\rho))^2+\|q\|_2^2$.
Since $\Var[p_{\rho}] < \tau \cdot \E[p_\rho^2]$, it follows that $\|q\|_2^2 \leq \tau \cdot (p'(\rho))^2$.
Hence, by the concentration bound of Theorem~\ref{thm:deg-d-chernoff}, we obtain
$\Pr_{x'} [|q(\rho, x')| \geq |p'(\rho)|/4 ] \leq \eta/4$.
If the latter event fails to hold, we get that $|p(x)| = |p_{\rho}(x')| \geq (3/4) |p'(\rho)| \geq (1/2) \|p_{\rho}\|_2$,
as desired.

Combining the above, we get that
$\Pr_{x} \left[ |p(x)| \leq \tau \cdot \|p_{\rho}\|_2 \mid \rho \textrm{ is Good} \right] \leq \eta/8 \;.$
We therefore obtain
$$\Pr_{x} \left[ |p(x)| \leq \tau \cdot \|p_{\rho}\|_2 \right] \leq \Pr_x\left[\rho \textrm{ is not Good}\right] +
\Pr_{x} \left[ |p(x)| \leq \tau \cdot \|p_{\rho}\|_2 \mid \rho \textrm{ is Good} \right]
\leq  \tau +  \eta/8  < \eta/4 \;.$$
This completes the proof of Claim~\ref{claim:p-rho}.
\end{proof}

To complete the proof of the lemma, we let $S$ be the set of all coordinates appearing in the decision tree. We note that $|S| \ll 2^{\textrm{depth}(\T)} = 2^{\poly_d(1/\eta)}$.
What remains to show is that $\|p_{x_S}\|_2$ is small for almost all $x\in D$. We do this by showing that with high probability
$\|p_{x_S}\|_2$ is not much larger than $ \|p_{\rho}\|_2$.

Recall that $S$ is the set of variables that appear in the decision tree, hence contains the variables
fixed by any restriction $\rho$ in any root to leaf path.
Note that for any such restriction $\rho$, we have that
$$\E_{x_S}\left[\|p_{x_S}\|_2^2 \mid \rho\right] =
\E_{x_S}\left[\E_{x'}\left[p(x_S, x')^2 \mid \rho\right]\right]=
\E_x \left[ p(x)^2 \mid \rho \right] =\|p_{\rho}\|_2^2 \;.$$
Therefore, by Markov's inequality, it follows that
\begin{equation} \label{eqn:markov}
\Pr_x \left[ \|p_{x_S}\|_2^2 > (4/\eta) \cdot \|p_{\rho}\|_2^2 \right] \leq \eta/4 \;.
\end{equation}
Inequality \eqref{eqn:markov}, Claim~\ref{claim:p-rho}, and a union bound
give that
$$\Pr_x \left[ \|p_{x_S}\|_2 > \poly_d(1/\eta) \cdot |p(x)|  \right] < \eta/2\;.$$
This competes the proof of Lemma~\ref{lem:small-norm-restr}.
\end{proof}

Combining Claim~\ref{claim:small-p-in-diff} and Lemma~\ref{lem:small-norm-restr},
we get that that for all but $\eta \cdot 2^n$ points in $D$ we have that
$$\delta/\eta > |p(x)| > \poly_d(\eta) \cdot \|p_{x_S}\|_2 \;.$$
Therefore, we have shown the following:

\begin{corollary} \label{cor:small-p-xs}
There exists a set of coordinates $S \subset [n]$ of size $|S| \leq 2^{\poly_d(1/\eta)}$
such that the set $$\widetilde{D} \eqdef \left\{ x \in D: \|p_{x_S}\|_2 = O(\delta/\poly_d(\eta)) \right\}$$
has $\Pr_x[\widetilde{D}] > \Pr_x[D] -\eta.$
\end{corollary}

We will construct a degree-$d$ multilinear polynomial $r$ that correlates with $p$ and is such that
$r(x) = 0$ for all $x \in \widetilde{D}$. This will complete the proof of
Proposition~\ref{prop:degd-poly-disagreement-region}, since
$\Pr_x \left[(f(x) \neq g(x)) \cap (r(x) \neq 0)\right] \leq \Pr_x[ D\setminus \widetilde{D}] < \eta$.
To do so, we will leverage the structural information provided by Corollary~\ref{cor:small-p-xs}. Let $S$ be the
set of coordinates satisfying Corollary~\ref{cor:small-p-xs}. We rewrite $p(x) = p_{x_S}(x')$ and note
that, for each fixed $x_S$, $p_{x_S}(x')$ is a degree-$d$ multilinear polynomial in $x'$.
We can view $p$ as a degree-$d$ multilinear polynomial in $x_S$ that returns degree-$d$ multilinear
polynomials on $x'$. Let $v_0, v_1, \ldots, v_{R}$, where $R = O(|S|^d)$,
be an orthonormal basis of the image $\mathrm{Im}(p_{x_S})$.
Then we can write
\begin{equation} \label{eqn:p-final}
p_{x_S}(x') = \sum_{0\leq j \leq R,\ T \subseteq S,\ |T| \leq d} \alpha_{T, j} \cdot \chi_T(x) \cdot \  v_j(x') \;.
\end{equation}
We note that the functions
$\left\{\chi_T(x) \cdot \  v_j(x') \right\}_{T \subseteq S,0 \leq j \leq R }$ are orthonormal.

The idea of the rest of the proof is as follows. We note that if the $\alpha_{T,j}$'s were all integers, we could use $r=p$ and would be done. This is because \emph{any} $p_{x_S}(x')$ would be an integer linear combination of the $v_j$'s, and thus would have integer $\ell_2$-norm. On the other hand, for $x\in \widetilde{D}$, it must be the case that $p_{x_S}$ has small $\ell_2$-norm. The only way that these can simultaneously hold is if $p_{x_S}=0$. In particular, this would imply that $p_{x_S}=0$ (and thus that $p(x)=0$) for all $x\in \widetilde{D}$.

In order to prove this result for more general $p$, we show that some multiple of $p$ can be \emph{approximated} by a polynomial where all of the $\alpha_{T,j}$'s are integers. To do so, we will need to make use of the following fact from the theory of Diophantine approximation:
\begin{fact} \label{fact:mult-int}
Let $N \in \Z_+$ and $w \in \R^N$. For any $\gamma \in (0, 1/2)$ there exists $t \in [1, O(1/\gamma)^{N}]$ such that
$$t \cdot w \in (\Z+(-\gamma, \gamma))^N \;.$$
\end{fact}
\begin{proof}
Let $\gamma > 1/m$ for some integer $m = O(1/\gamma)$.
We show in fact that $t$ can be taken to be an integer.

Partition $[0,1]^N$ into $m^N$ subcubes of side length $1/m$ in each dimension.
For each integer $k$ from $0$ to $m^N$, sort $kw\pmod{1}$ into the appropriate subcube.
Since there are $m^N+1$ values of $k$ and only $m^N$ subcubes,
by the pigeonhole principle, there must be $k\neq k'$ so that $kw\pmod{1}$
 and $k'w\pmod{1}$ fall in the same subcube. Without loss of generality, $k\geq k'$,
 and we let $t= k-k'$. It follows that the coordinates of $tw=kw-k'w$
 are all in $\Z+[-1/m,1/m] \subset \Z+(-\gamma,\gamma)$.
\end{proof}

We apply Fact~\ref{fact:mult-int} to the vector $(\alpha_{T, j})$
defined by the coefficients of the polynomial $p$ in \eqref{eqn:p-final}. We have that
$N = O(|S|^{2d})$ and we set $\gamma \eqdef 1/N^2$. It follows that there exists $t < 2^{2^{\poly_d(1/\eta)}}$
such that all $t \cdot \alpha_{T, j}$ are within an additive $1/N^2$ of
being integers. We can thus write
$$t \cdot p(x_S, x') = r(x_S, x') + e(x_S, x') \;,$$
where $r$ and $e$ are degree-$d$ multilinear polynomials
$$
r_{x_S}(x') = \sum_{0\leq j \leq R,\ T \subseteq S,\ |T| \leq d} \beta_{T, j} \cdot \chi_T(x) \cdot \  v_j(x'),
$$
and
$$
e_{x_S}(x') = \sum_{0\leq j \leq R,\ T \subseteq S,\ |T| \leq d} \gamma_{T, j} \cdot \chi_T(x) \cdot \  v_j(x'),
$$
where $\beta_{T,j}\in\Z$ and $|\gamma_{T,j}|<1/N^2$ for all $T,j$. Note that this implies that the sum of the $|\gamma_{T,j}|$ is at most $1/N$. It follows that for any value of $x_S\in\{\pm 1\}^S$ it holds $\|e_{x_S}\|_2 < 1/N$. 
On the other hand, for any such $x_S$, $r_{x_S}$ is an integer linear combination of $v_j(x')$, 
and thus $\|r_{x_S}\|_2^2 \in \Z$. However, for $x\in \widetilde{D}$, we have that
$$
\|r_{x_S}\|_2 \leq \|t \cdot p_{x_S}\|_2 + \|e_{x_S}\|_2 \leq t \cdot O(\delta/\poly_d(\eta)) + 1/N \;.
$$
Therefore, if $O(t \cdot \delta/\poly_d(\eta)) < 1/2$ (which holds for $1/\eta$ a sufficiently small polynomial in $\log\log(1/\delta)$), 
we have that for all $x\in \widetilde{D}$ it holds
$$
\|r_{x_S}\|_2 < 1 \;.
$$
However, since $\|r_{x_S}\|_2^2$ is an integer, this can only hold if $r_{x_S}=0$. Therefore, for all $x\in\widetilde{D}$, 
we have that $r(x)=0$, establishing (i).

We can now prove part (ii) of Proposition~\ref{prop:degd-poly-disagreement-region},
i.e., that $\E_x[p(x) r(x)] \neq 0$.
From the definition of $r$, we have that
$(1/t) r(x) = p(x)-e(x)/t.$ We note that $e_{x_S}$ always has $\ell_2$-norm at most $1/N$. 
It follows that $\|e(x)/t\|_2 < 1/2$.

Therefore,
$$\E_x[p(x) (1/t) r(x)] =  \E_x[p^2(x)] - \E_x[p(x)e(x)/t]  \geq 1 -1/2 > 0 \;,$$
where we used the assumption that $\E_x[p^2(x)] = \|p\|_2^2 = 1$ and
the Cauchy-Schwartz inequality.
This establishes (ii), and completes the proof of Proposition~\ref{prop:degd-poly-disagreement-region}.

\end{proof}

\subsection{Warm-Up: Completing the Proof for Degree-$2$ PTFs} \label{ssec:deg2}

In this subsection, we complete the proof of Theorem \ref{thm:chow-d-struct} 
for degree-$2$ polynomial threshold functions.

The high-level idea of the proof is as follows:
Given a pair $f, g$ contradicting our desired statement, by iteratively applying 
Proposition \ref{prop:degd-poly-disagreement-region} we can find many degree-$2$ polynomials 
in which almost all of the disagreements between $f$ and $g$ vanish. 
We would like to use this fact in order to reach a contradiction, 
by showing that there will be not enough points in the disagreement region. 
Unfortunately, this does not necessarily suffice, as it is possible to have 
many degree-$2$ polynomials that have joint zeroes at a substantial number of points. 
However, we show (see Proposition \ref{prop:linear-form-zero-set-deg2-subspace}) 
that given enough such degree-$2$ polynomials, we can find a single \emph{linear} 
polynomial that vanishes on almost all of these discrepancies. From there, we can 
restrict to the hyperplane defined by this polynomial yielding a new pair of functions 
with greater discrepancy and repeat this process.

We now proceed with the formal proof.
We start by showing that for any subspace of
degree-$2$ polynomials whose dimension is sufficiently large,
there exists a non-trivial linear function that vanishes
on almost all of its zero set.

\begin{proposition} \label{prop:linear-form-zero-set-deg2-subspace}
Fix $0< \eta < 1$. There exist positive integers $N_0 = N_0(\eta) = \Theta(1/\eta^3)$
and $N_1 = N_1(\eta) = \Theta(1/\eta^7)$ such that the following holds:
Let $V$ be any subspace of degree-$2$ $n$-variable polynomials of dimension $\dim(V) > N_1$
and $Z(V) \eqdef \{x \in \bn: p(x) = 0 \textrm{ for all } p \in V  \}$.
Then there exists a linear function $L: \R^n \to \R$, not identically zero, with $|\supp(L)| \leq N_0$
such that $\Pr_x [Z(V) \cap L(x) \neq 0] \leq \eta$.
\end{proposition}
\begin{proof}
The proof proceeds in two stages: 
First, via Corollary \ref{LOCor}, we show that there is a small set $S$ 
of coordinates so that all degree-$2$ terms of any polynomial in $V$ have a variable in $S$. 
Next, we consider the set of points in $\{\pm 1\}^S$ which define sub-cubes 
with a reasonably large number of joint zeroes. We show that either this set is contained 
in a linear subspace or that $\dim(V)$ must be bounded.

We can assume that $\Pr_x[Z(V)] > \eta$, otherwise there is nothing to prove.
Therefore, for any $p \in V$ it holds that $\Pr_x [p(x) = 0] > \eta$.
This implies the existence of structure on the coefficients of $p$,
which can be formalized using Theorem~\ref{thm:LO}, a degree-$d$ version
of the classical Littlewood-Offord lemma~\cite{LO:43, Erd:45}.

We start with the following claim:

\begin{claim} \label{claim:small-set-intersecting-all-terms}
There exists a set $S$ of at most $N_0 = \Theta(\eta^{-3})$ coordinates such that {\em for all}
$q \in V$ each degree-$2$ term of $q$ has a coordinate in $S$.
\end{claim}
\begin{proof}
We let $p$ be a generic element of $V$. In particular, $p_T$ should be non-zero for all $T$ for where $q_T\neq 0$ for any $q\in V$. Applying Corollary \ref{LOCor} to $p$, we find a set of $O(\eta^{-3})$ coordinates so that each non-vanishing degree-$2$ term of $p$ contains a variable in $S$. Since $p$ is generic, this implies that every non-vanishing degree-$2$ term of every $q$ in $V$ also has a variable in $S$. This completes our proof.
\begin{comment}
We will use the contrapositive of Theorem~\ref{thm:LO} for $d=2$.
Specifically, for any fixed $p \in V$, since $\Pr_x [p(x) = 0] > \eta$,
the rank of $p$ is at most $N'_0 = \Theta(\eta^{-3})$. Therefore, if we take
{\em any} $N'_0$ non-zero degree-$2$ terms of $p$, they must share a variable.
In turn, this implies that there exists a set $S$ of $N_0 = \nnew{2}N'_0$ coordinates
such that each degree-$2$ term of $p$ contains a coordinate in $S$.

We say that a polynomial $p \in V$ is {\em generic} if $p$ has a non-zero $x_ix_j$-coefficient
if there exists {\em any} $q \in V$ that has a non-zero $x_i x_j$ coefficient.
\nnew{Since $V$ is a vector space of degree-$2$ polynomials, it is easy to see
that it contains a generic polynomial. (In particular, if $p_1 \in V$ has zero $x_ix_j$-coefficient
and $p_2 \in V$ has non-zero $x_ix_j$-coefficient, then there exists
$\lambda \in \R$ such that $p_1+\lambda q_2 \in V$ contains all the non-zero coefficients of
$p_1$ and has non-zero $x_i x_j$ coefficient. By iterating this process, we can construct a generic polynomial $p \in V$.)}
By applying the aforementioned argument to a generic polynomial $p \in V$, Claim~\ref{claim:small-set-intersecting-all-terms} follows.
\end{comment}
\end{proof}

So far we have shown that any polynomial $p \in V$ is of the form
$$p(x) = p_{\mathrm{junta}}(x_S)+p_{\mathrm{cross}}(x_S, x') \;,$$
where $p_{\mathrm{junta}}(x_S)$ is a degree-$2$ polynomial in $x_S$
and $p_{\mathrm{cross}}(x_S, x')$ is linear in $x'$.

Let $V'$ be the subspace of $V$
consisting of polynomials with {\em no} non-zero terms using only
coordinates in $S$. That is, any polynomial $p \in V'$ is of the form:
$$p(x) = p_{\mathrm{cross}}(x_S, x') \;.$$
Note that
$$m \eqdef \dim(V') \geq \dim(V) - |S|^2 >N_1 - |S|^2 = \Omega(\eta^{-7}) \;,$$
where we used that $\dim(V) \geq N_1$ and $|S| \leq N_0.$
Let $Z(V') \eqdef \{x \in \bn: p(x) = 0 \textrm{ for all } p \in V'  \}$.
Since $V' \subseteq V$, it follows that $Z(V) \subseteq Z(V')$.
Therefore, $\Pr_x[Z(V')] \geq \Pr_x[Z(V)] > \eta$.

We can view each $p \in V'$ as an affine linear function from $x_S$
to the set of linear functions on $x'$. That is, for any fixed $x_S$,
the space $V'_{x_S} = \{ p(x_S, x'), p \in V' \}$ is a subspace of linear functions
on $x'$.

We will need the following basic fact
(see, e.g., Lemma~1 in~\cite{Goldberg:06b}):
\begin{fact} \label{fact:affine-subspace}
Let $W$ be a subspace of linear functions on $x$. Then
$\Pr_{x}[L(x)=0 \textrm{ for all } L \in W] \leq 2^{-\dim(W)}$.
\end{fact}
Let
$$G \eqdef \left\{ x_S \in \bits^{|S|} \textrm{ such that } \Pr_{x'} [p(x_S, x')=0 \textrm{ for all } p \in V'] > \eta \right\} \;.$$

%We observe that $G \neq \emptyset$. Indeed, we have that $\eta< \Pr_x[Z(V')]  = \E_{x_S}[ \Pr_{x'} [Z(V'_{x_S})]]$, which implies that there exists $x_S \in \bits^{|S|}$ such that $\Pr_{x'} [Z(V'_{x_S})] > \eta$.

Let $p_1, \ldots, p_m$ be a basis of $V'$.
Since, for any fixed $x_S$, each $p_i (x_S, x')$ is a linear function in $x'$,
by Fact~\ref{fact:affine-subspace} we get that for any fixed $x_S \in G$ we have
$$\dim(V'_{x_S}) = \dim(\mathrm{span}\{ p_i(x_S, x'), i \in [m]\}) < \log_2(1/\eta) \;.$$

We now establish the following claim:
\begin{claim} \label{claim:g-proper-affine}
The set $G$ lies in a {\em proper} affine linear subspace of $\bits^{|S|}$.
\end{claim}
\begin{proof}
For the sake of contradiction, suppose that there exist $y_1, \ldots, y_R \in G$,
with $R = |S|+1$, whose affine span is $\bits^{|S|}$.
Then every $x_S \in \bits^{|S|}$ can be written
as an affine linear combination of the $y_i$'s.
Since the $p_i(x_S, x')$'s are linear functions for every fixed $x_S$, it follows
that for every fixed $p_i$, $p_i(x_S, x') \in \mathrm{span}(\{ p_i(y_j, x'), j \in [R] \})$.
Therefore, for all $x_S \in  \bits^{|S|}$ and $p_i$ we have that
$$p_i(x_S, x') \in \mathrm{span}(\{ p_i(y_j, x'), j \in [R], i \in [m] \})=:U \;.$$
Note that $\dim(U) \leq R \cdot \log(1/\eta) = (|S|+1) \log(1/\eta)$.
Since all $p_i(x_S, x')$'s are affine linear functions from $x_S$ to $U$,
it follows that
$$\dim(V') \leq (|S|+1) \dim(U) \leq  (|S|+1)^2 \log(1/\eta)  = \tilde{O}(\eta^{-6}) \;,$$
which leads to the desired contradiction, since $m = \Omega(\eta^{-7})$.
This completes the proof of Claim~\ref{claim:g-proper-affine}.
\end{proof}

By Claim~\ref{claim:g-proper-affine}, it follows that there is a non-zero linear function $L: \R^{|S|} \to \R$
such that $L(x_S) = 0$ for all $x_S \in G$. %Since $G$ does not span $\bits^{|S|}$, it follows that $L \not\equiv 0$.
We can trivially extend $L$ to $\R^n$ by adding zero
coefficients for the remaining coordinates, i.e., $|\supp(L)| \leq |S| \leq N_0$.
It remains to argue that
$$\Pr_{x} \left[ Z(V) \cap (L(x) \neq 0) \right] \leq \eta \;.$$
Indeed, we can write
\begin{eqnarray*}
\Pr_{x}\left[Z(V) \cap (L(x) \neq 0)\right]
&\leq& \Pr_{x}\left[Z(V')\cap (L(x) \neq 0)\right] \\
&\leq& \Pr_{x}\left[Z(V') \cap (x_S \in \bar{G})\right] \\
&\leq&  \Pr_{x = (x_S, x')}\left[ Z(V') \mid x_S \in \bar{G} \right] \\
&\leq& \eta \;,
\end{eqnarray*}
where the first inequality holds since $Z(V) \subseteq Z(V')$,
the second follows from the definition of $L$ and the last inequality
from the definition of $G$.
This completes the proof of Proposition~\ref{prop:linear-form-zero-set-deg2-subspace}.
\end{proof}

We now have the necessary ingredients to complete the proof of our main result for $d=2$.
We prove the following proposition:

\begin{proposition} \label{lem:recursion-final}
Fix $0< \eps, \delta <1$.
Let $f(x) = \sgn(p(x))$ be a degree-$2$ PTF and $g:\bn \to \bits$ be a Boolean function
such that $\dist(f, g) \geq \eps$ and $\chow_2(f, g) \leq \delta$. Then, for all $m \geq 1$, there
exists a degree-$2$ PTF $f'_m$ and a Boolean-valued function $g'_m$ such that
$\dist(f'_m, g'_m) \geq (3/2)^m \cdot \epsilon$ and
$\chow_2(f'_m, g'_m) \leq \tilde{O}((\log^{\ast})^{(m)}(1/\delta)^{-1/7})$, so long as the latter term is less than $(3/2)^m \cdot \epsilon$.
\end{proposition}

The $d=2$ case of Theorem~\ref{thm:chow-d-struct} follows immediately from
Proposition~\ref{lem:recursion-final}. Indeed, if $1/\eps = 2^{o(\log^{\ast \ast} (1/\delta))}$,
then we obtain a contradiction by setting $m = \Omega(\log(1/\epsilon))$ in Proposition~\ref{lem:recursion-final}.

In the rest of this section, we give the proof of Proposition~\ref{lem:recursion-final}.
The proof follows by induction on $m$, where the base case
is given in the following lemma:

\begin{lemma} \label{lem:recursion-first-step}
Let $f(x) = \sgn(p(x))$ be a degree-$2$ PTF and $g:\bn \to \bits$ be
such that $\dist(f, g) \geq \eps$ and $\chow_2(f, g) \leq \delta$. There exists a degree-$2$ PTF
$f'$ and a Boolean function $g'$ such that $\dist(f', g') \geq (3/2) \cdot \epsilon$ and
$\chow_2(f', g') \leq \tilde{O}(\log^{\ast}(1/\delta)^{-1/7})$.
\end{lemma}

The basic idea of the proof is quite simple. By repeatedly applying Proposition~\ref{prop:degd-poly-disagreement-region}, we find many quadratic polynomials which vanish on almost all of the points of disagreement between $f$ and $g$. From there, we apply Proposition \ref{prop:linear-form-zero-set-deg2-subspace} to produce a single linear polynomial which vanishes on almost all disagreements. Restricting ourselves to this hyperplane, we can set one of our variables as a linear function of the others, and reduce to an $(n-1)$-dimensional cube without substantially changing the number of disagreements.

One issue with the above strategy is that we need to ensure that the repeated applications of Proposition~\ref{prop:degd-poly-disagreement-region} produce linearly independent polynomials. This can be achieved by modifying $p$ to be orthogonal to the previously found polynomials (and modifying $f$ and $g$ appropriately) so that the correlation condition in Proposition~\ref{prop:degd-poly-disagreement-region} implies that the polynomial produced will be new.

\begin{comment}
To prove Lemma~\ref{lem:recursion-first-step}, we will need
the following lemma, which shows that if $f$ and $g$
have large distance and small Chow distance,
then there is a non-trivial linear function
that vanishes in almost all of the disagreement region:
\end{comment}

We begin this program by showing that we can find 
a single linear polynomial that vanishes on almost all of the disagreements:

\begin{lemma} \label{lem:linear-form-vanishes-disagreement}
Fix $0< \delta <1$ and $\eta$ a sufficiently large multiple of  $(\log^{\ast}(1/\delta))^{-1/7}$.
Let $f(x) = \sgn(p(x))$ be a degree-$2$ PTF and $g:\bn \to \bits$ be
such that $\chow_2(f, g) \leq \delta$. Then there exits a non-trivial
linear function $L: \R^n \to \R$ with $|\supp(L)| \leq O(\log^{\ast}(1/\delta))$ such
that $\Pr_x[(f(x) \neq g(x)) \cap L(x) \neq 0] \leq \eta$.
\end{lemma}
\begin{proof}
Let $D = D(f, g)$ be the disagreement region between $f$ and $g$.
For a sequence of polynomials $\{p_i\}_{i \in [m]}$, we will denote by
$Z_m = \{ x \in \bn: p_i(x) = 0 \textrm{ for all } i\in [m] \}$ and by
$\bar{Z}_m$ its complement. We start by establishing the following claim:
\begin{claim} \label{claim:many-independent-polys-vanish}
For all $m \geq 1$, there exist linearly independent degree-$2$ polynomials $p_1, \ldots, p_m$ with $\eta_0=\delta$ and
$\eta_m = 1/\poly(\log\log(1/\eta_{m-1}))$ such that
$\Pr_x [D \cap \bar{Z}_m] \leq \eta_m$.
\end{claim}
\begin{proof}
By induction on $m$. The base case, ($m=0$) is trivial.
For the induction step, suppose that $p_1, \ldots, p_m$ and $\eta_m$ exist such that the claim holds.
We will prove that there exists $p_1, \ldots, p_{m+1}$ and $\eta_{m+1}$ satisfying the claim statement.
Consider the polynomial $p'$ defined by
$p' = p - \mathrm{proj}(p, \mathrm{span}(p_1, \ldots, p_m))$.
Now note that the degree-$2$ PTF $f'(x) = \sgn(p'(x))$ satisfies
$f'(x) = f(x)$ for all $x \in Z_m$. Let
\[
g'(x)= \left\{
\begin{array}{ll}
      f'(x) \;,  & x \in \bar{Z}_m \\
      g(x) \;,  & x \in Z_m\\
\end{array}
\right.
\]
Note that $f'(x) - g'(x) = f(x) - g(x)$,  $x \in Z_m$, and $f'(x) - g'(x) = 0$, otherwise.
This means that $f-g = f'-g'$ except on $\eta_m$-mass of $D$.
Therefore, $\chow_2(f', g') = \chow_2(f, g)+\tilde{O}(\eta_m) = \tilde{O}(\delta+\eta_m)$
and $\dist(f', g') \geq \eps-\eta_m \approx \eps$.

Therefore, we can apply
Proposition~\ref{prop:degd-poly-disagreement-region} to the pair $f', g'$.
We thus obtain that there exists a polynomial $p_{m+1}$ that correlates with $p'$
such that on all but $\eta_{m+1} = 1/\poly (\log \log (1/\eta_m))$ mass of $D(f', g')$
we have $p_{m+1}(x)=0$.

Since $D(f', g')$ and $D(f, g)$ are within $\eta_m$ mass of each other and $\eta_m \ll \eta_{m+1}$,
we obtain that all but at most $O(\eta_{m+1}) \cdot 2^n$ points of $D(f, g)$ satisfy that
$p_1(x) = \ldots = p_{m+1}(x) = 0$.

By the conditions of Proposition~\ref{prop:degd-poly-disagreement-region}, we have that
$\E_x[p_{m+1}(x) p'(x)] \neq 0$. But by the definition of $p'$, we have that
$\E_x[p_i(x) p'(x)]=0$ for all $i \leq m$. Thus, $p_{m+1} \not\in \mathrm{span}(\{p_i, i \in [m] \})$
and the inductive step is complete. This completes the proof of Claim~\ref{claim:many-independent-polys-vanish}.
\end{proof}

We are now ready to complete the proof of the lemma.
We select $N$ such that $ \log^{[O(N^7)]} (1/\delta) >N$.
By Claim~\ref{claim:many-independent-polys-vanish}, 
we obtain a set of linearly independent polynomials $p_1, \ldots, p_{N^7}$
such that all but $(1/N)\cdot 2^n$ points of $D(f, g)$ satisfy
$p_1(x) = \ldots = p_{N^7}(x)=0$. By Proposition~\ref{prop:linear-form-zero-set-deg2-subspace}, 
there exists a non-zero linear form $L:\R^n \to \R$ with support $O(N^7)$ such that 
for all but $O(1/N) \cdot 2^n$ points $x$ of $D(f, g)$ it holds $L(x)=0$. 
Note that the probability mass of points where $L(x) = 0$ is $O(1/\log^{\ast}(1/\delta)^{1/7})$.
This completes the proof of Lemma~\ref{lem:linear-form-vanishes-disagreement}.
\end{proof}

We are now prepared to prove Lemma~\ref{lem:recursion-first-step}.
\begin{proof}[Proof of Lemma~\ref{lem:recursion-first-step}.]
By Lemma~\ref{lem:linear-form-vanishes-disagreement}, 
there exists a linear form $L$, not identically zero, with support
of size $O(\log^{\ast}(1/\delta))$ such that all but $O(1/\log^{\ast}(1/\delta)^{1/7})$ mass of
points $x$ in $D(f, g)$ has $L(x) = 0$.

Since in most of the disagreement region $D(f, g)$ we have $L(x) = 0$, 
it follows that, for each such point $x$, we can express one of the variables in $L(x)$
as a function of the others. Specifically, we can write $L(x_1, \ldots, x_n) = x_1 - L'(x_2, \ldots, x_n)$.

We now consider the function $g_0(x)$ which equals $f(x)$ if $L(x) \neq 0$ and $g(x)$
otherwise. Note that $\dist(f, g_0) \geq 3\eps/4$ and
$\chow_2(f, g_0) \leq \tilde{O}(\log^{\ast}(1/\delta)^{-1/7})$.

Recall that $f(x) = \sgn(p(x))$. We consider the function $f'$ which is a degree-$2$
PTF over $x_2, \ldots, x_n$ defined as follows:
$f'(x_2, \ldots, x_n) = \sgn(p(L'(x_2, \ldots, x_n), x_2, \ldots, x_n))$
and the Boolean function $g'$ defined by
$g'(x_2, \ldots, x_n) = f'(x_2, \ldots, x_n)$ if $L'(x_2, \ldots, x_n) \notin \{ \pm 1 \}$;
and $g'(x_2, \ldots, x_n) = g(L'(x_2, \ldots, x_n) ,x_2, \ldots, x_n)$ otherwise.
Then we have that $\dist(f', g') = 2 \dist(f, g_0) \geq 3\eps/2$
and $\chow_2(f', g') \leq \chow_2(f, g_0) \leq \tilde{O}(\log^{\ast}(1/\delta)^{-1/7})$.
This completes the proof of Lemma~\ref{lem:recursion-first-step}.
\end{proof}

The proof of Proposition \ref{lem:recursion-final} is now immediate. 
Given the existence of $f_m'$ and $g_m'$, we obtain $f_{m+1}'$ and $g_{m+1}'$ 
by applying Lemma \ref{lem:recursion-first-step}.

\subsection{Completing the Proof for Degree-$d$ PTFs } \label{ssec:degd}

The high-level approach of our degree-$d$ generalization is similar as in the degree-$2$ case.
First, we make use of Proposition~\ref{prop:degd-poly-disagreement-region}, 
showing that if we have a small degree-$d$ Chow distance, then all but a tiny number
of discrepancies must lie on the zero set of a degree-$d$ polynomial.
From there, we can iterate this to show that our discrepancies must lie 
on the intersection of \emph{many} degree-$d$ polynomials. We would ideally 
like to be able to generalize Lemma~\ref{lem:linear-form-vanishes-disagreement}
and deduce that once there are enough of them, we will be forced to lie on a linear subspace,
but this is too ambitious in general. Instead, we will use many degree-$d$ polynomials 
to force our discrepancies to lie on the zero set of a single degree-$(d-1)$ polynomial. 
From there, we will go back to gathering degree-$d$ polynomials until we have enough 
to force a second degree-$(d-1)$ polynomial; and eventually we will have enough of those 
to force a degree-$(d-2)$ polynomial, and so on. To keep track of the general state of this recursion, 
we will need to maintain an {\em ideal} of polynomials which vanish on our discrepancies.

We start by recalling the definition of an ideal:

\begin{definition} \label{def:ideal}
An ideal (in $\R[x_1,\ldots,x_n]$) is a set $I$ of polynomials so that
\begin{enumerate}
\item For any $p,q\in I$, we have $p+q\in I$.
\item For any $p\in I$ and $a\in \R[x_1,\ldots,x_n]$, we have that $a\cdot p \in I$.
\end{enumerate}
\end{definition}
\noindent Note that an ideal $I$ is also a vector subspace of the space of all polynomials.

\paragraph{Notation.}
We will use the following notation in the rest of the proof.
For polynomials $p_1,p_2,\ldots,p_m$, we consider the ideal 
$I =\left\{ \sum_{i=1}^m a_i p_i:a_i \in \R[x_1,\ldots,x_n] \right\}$, which we will denote by $(p_1,p_2,\ldots,p_m)$. 
For two ideals, $I$ and $J$, we have that the set of polynomials $\{p+q:p\in I, q\in J\}$ is another ideal, 
which we will denote $I+J$. We say that two polynomials $p$ and $q$ are congruent modulo $I$ if $p-q\in I$. 
We note that this is an equivalence relation on polynomials.

We will also care about the points of the hypercube on which an ideal vanishes.

\begin{definition}
For a set ${\cal P}$ of polynomials, let $Z({\cal P})$ denote the set
$\{x\in \{\pm 1\}^n: p(x) = 0 \textrm{ for all }p\in {\cal P}\}$.
\end{definition}

Of particular interest is the case where ${\cal P}$ is an ideal $I$.
In particular, we note that $Z((p_1,\ldots,p_m)) = \{x\in \{\pm 1\}^n: p_i(x)=0 \textrm{ for all } 1\leq i\leq m\}$
and that $Z(I+ J) =Z(I)\cap Z(J)$.

%We begin with our generalization of [REFERENCE FOR LEMMA 1 IN PHOTOS]

%[PROOF]

We now proceed to generalize Proposition~\ref{prop:linear-form-zero-set-deg2-subspace} to the degree-$d$ setting.
It turns out that this generalization is somewhat more complicated. We prove:

\begin{proposition}\label{lowerDegreePolyProp}
For every positive integer $d$ and every $\eta>0$ there exists an $m=\eta^{-d^{O(d)}}$ so that the following holds:
For $I$ an ideal containing $x_i^2-1$ for all $i \in [n]$, and $V$ a vector subspace of the space of degree at most $d$ polynomials with $I\cap V = \{0\}$, and $\dim(V) > m$, there exists a degree at most $d-1$ polynomial $P\not\in I$ so that $\P[Z(I)\cap Z(V)\backslash Z(P)]\leq \eta$.
\end{proposition}
\begin{proof}
Firstly, we note that we may assume that $\P[Z(I)\cap Z(V)]>\eta$,
or the result is trivial (for example, we can take $P=1$, as $1$ cannot be in $I$ or $I$ would have to contain $V$). Secondly, let $V$ have a basis $p_1,p_2,\ldots,p_t$. We note that if we find
another set of degree-at-most-$d$ polynomials $q_1,q_2,\ldots,q_t$ with $p_i \equiv q_i \pmod{I}$
and let $V' =\span(\{q_i\})$, then $Z(I)\cap Z(V) = Z(I)\cap Z(V')$, $I\cap V' = \{0\}$
and $\dim(V')=\dim(V)$, and thus proving our statement for $V$ is equivalent to proving it for $V'$.
We call such $q_i$ an equivalent basis.

Our proof proceeds by proving the following lemma:
\begin{lemma}
Given any ideal $I$ and vector space $V$ and any integer $0\leq k \leq d$
either there exists a set $S_k$ of $O(\eta^{-d^{2(d-k)}})$ many coordinates and an equivalent basis $q_i$ of $V$ so that every non-zero monomial of every $q_i$ contains at most $k$ coordinates not in $S_k$ or there exists a degree less than $d$ polynomial $P\not\in I$ so that $\P[Z(I)\cap Z(V)\backslash Z(P)]\leq \eta$.
\end{lemma}
We note that the $k=0$ case of this lemma immediately implies our result, as either such a $P$ exists, or (up to taking an equivalent basis) $V$ is a space of polynomials in $|S_0|$ variables, and thus has dimension at most $\sum_{i=0}^d \binom{|S_0|}{i}$. If $m$ is larger than this, we are done. It remains 
to prove this lemma.
\begin{proof}
The proof proceeds by backwards induction on $k$ and is reminiscent of the proof of Proposition \ref{prop:linear-form-zero-set-deg2-subspace}. In particular, given an $S_k$, we think of our polynomials as being degree-$d$ polynomials from $\{\pm 1\}^{S_k}$ to degree-$k$ polynomials in the remaining variables. We consider the set of $x$ in $\{\pm 1\}^{S_k}$ with a substantial number of points in $Z(V)\cap Z(I)$. By Corollary \ref{LOCor}, we note that for each such $x_S$, there is a small set of remaining coordinates $T_x$ so that all of the degree $k$ terms in $p_{x_{S_k}}$ have a variable in $T$. By an appropriate dimensionality argument, we show that in fact we can use only a single $T$ for all good $x$. From there, we show that it is either the case that some polynomial not in $I$ vanishes on all good $x$, or that the generators of $V$ can be reduced modulo $I$ to have no terms of degree more than $k-1$ in variables outside of $S_k\cup T$.

We proceed by backwards induction on $k$. In particular, for $k=d$, the result is trivial. Otherwise, assume that it holds for a given value of $k$. By the inductive hypothesis, either there exists a polynomial $P$, or there exists an appropriate set of coordinates $S_k$. By replacing $V$ by an appropriate equivalent basis, we may assume that the $p_i$ are multilinear and have all non-zero terms have at most $k$ coordinates not in $S_k$. Define
$$
G:= \{ x\in \{\pm 1\}^{S_k} : \pr(y\in Z(I)\cap Z(V)|y_{S_k}=x) > \eta\}.
$$
Note that all but an $\eta$-mass of the points of $Z(I)\cap Z(V)$ have $x_{S_k}\in G$.

Next, for $x\in \{\pm 1\}^{S_k}$, define $M(x)$ to be the vector of degree at most $(d-k)$-monomials in $x$. Note that if $M(x_0)$ is a linear combination of some $M(x_i)$, then for any polynomial $p$ of degree at most $d-k$, $p(x)$ will be the same linear combination of the $p(x_i)$. Let $x_1,x_2,\ldots,x_s$ be points in $G$ so that $M(x_1),M(x_2),\ldots,M(x_s)$ spans all of the $M(x)$ for all $x\in G$. Note that we can take $s$ to be at most the dimension of the range of $M$, which is $O(|S_k|^d)$.

Let $q$ be a generic element of $V$. Let $q_i$ be the degree-$k$ polynomial obtained by setting of the $S_k$-coordinates of the input of $q$ to those of $x_i$. We note that $\pr(q_i(y)=0) \geq \eta$. By Corollary \ref{LOCor}, this implies that there is a set $T_i$ of at most $2^{O(d^2)}\eta^{-3}$ coordinates so that every non-zero, degree-$k$ monomial in $q_i$ has at least one coordinate in $T_i$. Since $q_i$ is a generic linear combination of the $p_j$, this must mean that each $p_j$ when restricted to $x_i$ must have all its degree-$k$ terms having a coordinate in $T_i$.

Let $S=S_k\cup \bigcup_{i=1}^s T_i$. We note that $|S|=O(\eta^{-d^{2(d-k+1)}})$. We will attempt to use $S$ for $S_{k-1}$.

Rewrite each $p_\ell(x)$ as $r_\ell(x)+\sum m_{\ell j}(x_{S^c})c_{\ell j}(x_S),$ where all the monomials in $r_\ell$ have at most $k-1$ coordinates not in $S$, each $m_{\ell j}$ is a monomial of degree-$k$ with only coordinates not in $S$, and $c_{\ell j}$ a polynomial with coordinates in $S$. Note that for any $x_i$ that $c_{\ell j}(x_i)=0$. For any $x\in G$ $c_{\ell j}(x)$ is a linear combination of the $c_{\ell j}(x_i)$ and is thus also zero. We now split into two cases:

\medskip

\textbf{Case 1:} $c_{\ell j}\in I$ for all $\ell$ and $j$.

In this case, $p_\ell \equiv r_\ell\pmod{I}$. Thus, taking the $r_\ell$ as the equivalent basis, it contains no monomial with more than $k-1$ coordinates not in $S$, so we are done.

\medskip

\textbf{Case 2:} Some $c_{\ell j}\not\in I$.

We note that $c_{\ell j}$ is a degree less than $d$ polynomial that is not in $I$. Furthermore, by the above, it vanishes on $G$. Thus, it vanishes on all but $\eta 2^n$ points of $Z(I)\cap Z(V)$. Thus, we can take $P=c_{\ell j}$.

This completes the inductive step, and proves our lemma.
\end{proof}
\end{proof}

In order to properly analyze the process of iteratively applying the above proposition, some work needs to be done in order to find the correct inductive statement. The following is the one that works conveniently:
\begin{proposition}
For any integers $d\geq k\geq 0$ there exists a function $h_{d,k}:\R^+ \rightarrow \R^+$ so that
$\lim_{\delta\rightarrow 0} h_{d,k}(\delta)=0$ and so that if $f$ is any degree-$d$ PTF, $g:\bn \to ]bits$ is 
any boolean function, and $I$ any ideal containing all $x_i^2-1$ so that $f(x)\neq g(x)$
only for $x\in Z(I)$ and where $\chow_d (f, g) <\delta$, then there exists a degree at most $k$ polynomial 
$P\not\in I$ so that $P(x)=0$ for all but an $h_{d,k}(\delta)$-mass of the points where $f(x)\neq g(x)$.
\end{proposition}
Before we prove this proposition, we note why it immediately implies Theorem~\ref{thm:chow-d-struct}.
In particular, taking $k=0$ and $I=(x_1^2-1,\ldots,x_n^2-1)$, this says that if $\chow_d (f, g) < \delta$, there exists a \emph{constant} function $P\neq 0$, so that all but an $h_{d,0}(\delta)$-mass of the disagreements of $f$ and $g$ have $P=0$ (which never happens).
In particular, this means that
$$
\chow_d (f, g)< \delta \Rightarrow \dist(f, g)< h_{d,0}(\delta).
$$

\begin{proof}
We proceed once again by backwards induction on $k$. We start with the base case of $k=d$. This will follow essentially from Proposition~\ref{prop:degd-poly-disagreement-region}. The idea is simple. Let $f(x)=\sgn(p(x))$. Let $I_d$ be the vector space of all degree-at-most-$d$ polynomials in $I$. Let $p'$ be the part of $p$ perpendicular to $I_d$, and let $f'(x) = \sgn(p'(x))$. Notice that since $p' \equiv p \pmod{I}$, we have that $f(x)=f'(x)$ for $x\in Z(I)$. Let
$$
g'(x) = \begin{cases}f'(x) & \textrm{if } x\not\in Z(I) \\ g(x) & \textrm{otherwise} \end{cases}.
$$
Note that $f(x)-g(x) = f'(x)-g'(x)$ and so the discrepancy sets are the same as are the Chow distances. Thus, it suffices to prove our result for $f'$ and $g'$. However, we know from
Proposition~\ref{prop:degd-poly-disagreement-region}
that for $h_{d,d}(\delta) = \eta = 1/\poly_d(\log\log(1/\delta))$, that there exists a polynomial $P$ of degree at most $d$ with $P$ not perpendicular to $p'$ so that all but an $\eta$-mass of the discrepancy of $f'$ and $g'$ lies in $Z(P)$. However, $p'$ is orthogonal to all of $I_d$, and therefore $P$ cannot be in $I$. This completes the proof when $k=d$.

For the inductive step, assume that there is such a function $h_{d,k}$ for some given value of $k$. Note that for any triple $f,g,I$ satisfying our hypotheses, there is a $P$ of degree at most $k$ where all but an $h_{d,k}(\delta)$-mass of the discrepancies lie on $Z(P)$. If we define $g'(x)$ to be $g(x)$ when $P(x)=0$ and $f(x)$ otherwise, note that all of the discrepancies of $f$ and $g'$ lie on $Z(I+(P))$. Furthermore since $g$ and $g'$ differ in hamming weight $h_{d,k}(\delta)$, we have that
$$
\chow_d (f, g') \leq \chow_d (f, g) + \chow_d (g, g') \leq \delta + O_d(h_{d,k}(\delta))^{1/3} := h'_{d,k}(\delta).
$$

We are now able to define $h_{d,k-1}$. In particular, let $m_{d}(\eta)$ be the implied constant from Proposition \ref{lowerDegreePolyProp}. We claim that we can take $h_{d,k-1}$ to be less than $\eta$ whenever $h'_{d,k}$ iterated on $\delta$ $m_{k}(\eta^2)$ times is less than $\eta^2$. In particular, suppose that this holds for $\delta$ and we have that $\chow_d (f, g)<\delta$. We would like to claim that we can find an appropriate $P$ of degree less than $k$ with all but an $\eta$ mass of the discrepancies between $f$ and $g$ lying on $Z(P)$.

Let $g=g_0$. We note that all discrepancies between $f$ and $g$ lie on $Z(I)$. By the above, there exists a degree-$d$ polynomial $P_1\not \in I$ and a $g_1$ with
$\dist(g, g_1)< h_{d,k}(\delta)$, $\chow_d (f, g_1)<\delta_1 = h'_{d,k}(\delta)$ and all discrepancies between $f$ and $g_1$ lying in $Z(I_1)$ where $I_1= I+(P_1)$. Iterating this, we can find
$g_t$ with $\dist(g_{t-1}, g_t)< h_{d,k}(\delta_{t-1})$,
$\chow_d (f, g_1)<\delta_t = h'_{d,k}(\delta_{t-1})$ and all discrepancies between $f$ and $g_t$
lying in $Z(I_t)$,  where $I_1= I+(P_1,P_2,\ldots, P_t)$. Take $t=m_d(\eta)$,
so that $\delta_t < \eta^2$. Note that since $P_i \not\in I+(P_1,\ldots,P_{i-1})$,
it must be the case that $V=\span(P_1,\ldots,P_t)$ has trivial intersection with $I$.
Therefore, applying Proposition \ref{lowerDegreePolyProp} to $I$ and $V$,
we find that there is a degree less than $k$ polynomial $P$ so that all but an $\eta^2$-mass
of the discrepancies between $f$ and $g_t$ (which in turn is all but an $\eta^2$-mass of the discrepancies between $f$ and $g$) lie on $Z(P)$.

This completes our inductive step and the proof.
\end{proof}

\section{Algorithmic Applications of Theorem~\ref{thm:chow-d-struct}} \label{sec:apps}
Our main structural result, Theorem~\ref{thm:chow-d-struct}, 
together with machinery developed in~\cite{TTV:09short, DeDFS14, DiakonikolasKS18a-nasty}
yields the first efficient algorithms for the degree-$d$ Chow parameters problem
and for learning Boolean degree-$d$ PTFs with malicious noise. As a corollary of the former result,
we also obtain the existence of low integer-weight approximations to degree-$d$ PTFs.
In this section, we describe these applications and explain how they are obtained
by combining Theorem~\ref{thm:chow-d-struct} with prior work.

\subsection{Degree-$d$ Chow Parameters Problem and Low Integer-Weight Approximation} \label{ssec:chow-apps}
We start by proving Theorem~\ref{thm:chow-d-alg}. 
To prove our theorem, we need an efficient algorithm that starts with (approximations to) the degree-$d$ 
Chow parameters of our degree-$d$ PTF $f$ and computes the coefficients 
of a degree-$d$ multilinear polynomial that approximately sign represents $f$. 
This can be done by known techniques, as follows from prior work~\cite{TTV:09short, DeDFS14}.

We first need to define the notion of a projection:
\begin{definition} \label{def:projection}
For $a \in \R$, we denote its projection to $[-1,1]$ by $P_1(a)$. 
That is, $P_1(a) = a$ if $|a| \leq 1$ and $P_1(a) = \sign(a)$, otherwise.
\end{definition}

Our algorithm will make essential use of a variant of degree-$d$ PTFs, which
we call degree-$d$ \emph{polynomial bounded functions} (PBFs):

\begin{definition} \label{def:pbf}
A function $g:\{-1,1\}^n \rightarrow [-1,1]$ is referred to as a degree-$d$ {\em polynomial bounded function} (PBF) 
if there exists a degree-$d$ multilinear polynomial $Q: \bn \to \R$ with 
$Q(x)  = \sum_{S\subseteq [n], |S| \leq d} Q_S \chi_S(x)$ such that
$g(x)  = P_1(Q(x))$. The vector of coefficients $(Q_S)_S$ is said to represent $g$.
\end{definition}

The following result shows that there exists an efficient algorithm which, given approximations
to the degree-$d$ Chow parameters of an arbitrary Boolean-valued function, outputs a degree-$d$
PBF with approximately these degree-$d$ Chow parameters:

\begin{theorem}[Degree-$d$ Chow Reconstruction] \label{thm:alg-chowd}
There exists a randomized algorithm {\tt Chow-d-Reconstruct} that for every Boolean function 
$f:  \bn \to \bits$, given $\xi>0,\delta>0$ and a vector $\vec{\alpha} = (\alpha_S)_{S \subseteq [n], |S| \leq d}$ 
such that $\|\mychows^d_f - \vec{\alpha}\|_2 \leq \xi$,
with probability at least $1-\delta$, outputs a degree-$d$ PBF $g$ such that 
$\|\mychows^d_f - \mychows^d_g\|_2 \leq 6 \xi$. 
The algorithm runs in time $\tilde{O}(n^{2d}) \poly(1/\xi) \log{(1/\delta})$.
Further, $g$ is represented by a weight-vector $ \lambda v $, where $\lambda \in \R$ and 
$v$ is an integer vector with $\|v\|_2^2 = O(n^{d}) \poly(1/\xi)$.
\end{theorem}

We remark that the condition on the weight vector $v$ given by Theorem \ref{thm:alg-chowd} 
is the key for the proof of Theorem \ref{thm:lowwt}. We note that Theorem \ref{thm:alg-chowd} follows directly
either from a more general in ~\cite{TTV:09short} (specifically, their Theorem~3.1; also see Theorem~16 of~\cite{DeDS17}),
or from a straightforward generalization of the $d=1$ algorithm for the same problem in~\cite{DeDFS14}.

\smallskip

The algorithm establishing Theorem~\ref{thm:alg-chowd} is a simple iterative algorithm running 
for $O(1/\xi^2)$ iterations, where in each iteration 
it estimates the degree-$d$ Chow parameters of the current hypothesis $h'$
within $\ell_2$ error $O(\xi)$. Here, each $h'$ is a degree-$d$ PBF, i.e., 
function of the form $h'(x)=P_1({\frac \xi 2} \cdot \sum_{S\subseteq [n], |S| \leq d} H_S \chi_S(x))$,
where the coefficients $H_S$ are integers whose absolute values sum to $O(1/\xi^2)$.

\begin{proof}[Proof of Theorem~\ref{thm:chow-d-alg}]
Given a vector $\vec{\alpha}$ such that $\Delta : = \| \vec{\alpha} - \mychows^d_f \|_2 \leq \chowallow(\eps, d)$, where
$f$ is the unknown degree-$d$ PTF to be learned, we proceed as follows: 
To construct the desired hypothesis $f^{\ast}$, we run algorithm {\tt Chow-dReconstruct}  (from Theorem~\ref{thm:alg-chowd}) on input
$\vec{\alpha}$. The algorithm runs in time $\poly(1/\Delta) \cdot \tilde{O}(n^{2d}) \cdot \log(1/\delta)$ and
outputs a degree-$d$ PBF $g$ such that with probability at least $1-\delta$ we have
$\chow_d(f,g) \leq 6 \Delta \leq 6 \chowallow(\eps, d)$.
By Theorem~\ref{thm:chow-d-struct}, we get that with probability at least $1-\delta$  
we have $\dist(f,g) \leq \eps/2$.  
(By setting the constants appropriately in the definition of $ \chowallow(\eps, d)$ above,
we can guarantee that the conclusion of Theorem~\ref{thm:chow-d-struct} is $\dist(f,g) \leq \eps/2$.) 
Writing the degree-$d$ PBF $g$ as $g(x) = P_1(Q(x))$, we now claim that
$f^\ast(x) = \sign(Q(x)$ has $\dist(f,f^\ast) \leq \eps.$  This holds because, for each input $x \in \bn$, 
the contribution that $x$ makes to to $\dist(f,f^\ast)$ is at most twice the contribution $x$ makes to $\dist(f,g)$. 
This completes the proof of Theorem~\ref{thm:chow-d-alg}.
\end{proof}

\noindent As a simple corollary, we obtain Theorem~\ref{thm:lowwt}.

\begin{proof}[Proof of Theorem~\ref{thm:lowwt}]
Let $f:\bn \to \bits$ be an arbitrary degree-$d$ PTF. 
We apply Theorem~\ref{thm:chow-d-alg}, for $\delta = 1/3$, and consider the degree-$d$ PTF, $f^{\ast}$, 
output by the algorithm. Note that the weights $v_i$ defining $f^{\ast}$ are identical to the weights of the 
PBF $g$ output by  the algorithm {\tt Chow-d-Reconstruct}.
It follows from (the proof of) Theorem~\ref{thm:chow-d-alg} that these weights are integers that satisfy 
$\littlesum_{S} v^2_S  = O\left(n^d \cdot \poly(1/\Delta)\right)$, where
$\Delta = \Omega ( \chowallow(\eps, d))$, and the proof is complete.
\end{proof}

\paragraph{Learning PTFs in the RFA Model.}
Ben-David and Dichterman~\cite{BenDavidDichterman:98} 
introduced the ``Restricted Focus of Attention'' (RFA) learning framework
to model the phenomenon of a learner having incomplete access to examples.
We focus here on the uniform-distribution ``$d$-RFA'' model.   In this setting
each time the learner is to receive a labeled
example, it first specifies a set $J \subseteq [n]$ of at most $d$ indices; 
then an $n$-bit string $x$ is drawn from the uniform distribution over $\bn$ and the learner is
given $(x_J, f(x))$.  So for each labeled example, the learner is only shown the bits of the
example indexed by $J$ along with the label. 

Note that learning in the $d$-RFA model is closely related to the degree-$d$ Chow parameters
problem. Indeed, a learning algorithm in the $d$-RFA model can only use the examples 
to estimate the degree-$d$ Chow parameters of the unknown target concept $f$ to any desired accuracy.  
This connection was established in Birkendorf et al.~\cite{BDJ+:98} who asked
the $d=1$ version of this question, i.e., whether LTFs can be learned in the uniform
distribution $1$-RFA model. 
For $d=1$, the structural results of Goldberg~\cite{Goldberg:06b} and Servedio~\cite{Servedio:07cc} established
information-theoretic upper bounds on the sample complexity of the problem.
The algorithmic results of~\cite{OS11:chow, DeDFS14} giving algorithms for the degree-$1$ Chow
parameters problem immediately imply efficient algorithms for learning LTFs in the uniform distribution 
$1$-RFA model.

As a direct consequence of Theorem~\ref{thm:chow-d-alg}, we obtain the first
efficient learning algorithm for learning degree-$d$ PTFs in the uniform distribution $d$-RFA model:

\begin{theorem} \label{thm:learnrfa} 
There is an algorithm which performs $\tilde{O}(n^{2d}) \cdot \poly(1/\chowallow(\eps, d)) \cdot \log({\frac 1 \delta})$ 
bit-operations and properly learns degree-$d$ PTFs  
to accuracy $\eps$ and confidence $1 - \delta$ in the uniform distribution $d$-RFA model.
\end{theorem}

Prior to our work, even for $d=2$, no sub-exponential in $n$ upper bound was known for this problem,
even with respect to sample complexity only.

\subsection{Malicious Learning of Boolean Degree-$d$ PTFs} \label{ssec:app-malicious}
In this section, we sketch how Theorem~\ref{thm:malicious} is obtained using 
Theorem~\ref{thm:chow-d-struct}. The results of this section follow directly 
from the recent work~\cite{DiakonikolasKS18a-nasty}. 
We provide a brief description here for the sake of completeness.

We focus on the efficient learnability of Boolean low-degree PTFs in the presence of {\em nasty noise},
a model of corruptions that strengthens malicious noise and has recently received
renewed attention motivated by robust high-dimensional statistics~\cite{DKKLMS16}.
In the {\em nasty noise model}~\cite{BEK:02}, an omniscient adversary can
arbitrarily corrupt a small constant fraction of both the unlabeled data points and their labels.
Formally, we have the following definition:

\begin{definition}[Nasty Noise Learning Model] \label{def:nasty-learning}
Let $\mathcal{C}$ be a class of Boolean-valued functions over $\R^n$, $D$ a distribution
over $\R^n$, and $f$ an unknown target concept $f \in \mathcal{C}$. For $0< \eps < 1/2$,
we say that a set $T$ of $m$ labeled examples is an {\em $\eps$-corrupted}
set of examples from $\mathcal{C}$
if it is obtained using the following procedure: First, we draw a set $S = \{ (x^{(i)}, y_i)\}$
of $m$ labeled examples, $1 \leq i \leq m$, where for each $i$
we have that $x^{(i)} \sim D$, $y_i = f(x^{(i)})$, and the $x^{(i)}$'s are independent.
Then an omniscient adversary, upon inspecting the set $S$,
is allowed to remove an  $\eps$-fraction of the examples
and replace these examples by the same number of arbitrary examples of its choice.
The modified set of labeled examples is the $\eps$-corrupted set $T$.
A learning algorithm in the nasty noise model is given as input an $\eps$-corrupted set of examples
from $\mathcal{C}$ and its goal is to output a hypothesis $h$ such that with high probability
the error $\Pr_{x \sim D} [h(x) \neq f(x)]$ is small.
\end{definition}

It should be noted that the nasty model generalizes a number of well-studied noise models,
including the malicious noise model~\cite{Valiant:85, KearnsLi:93}
and the agnostic (adversarial label noise) model~\cite{Haussler:92, KSS:94}.

A first qualitative goal is to design polynomial-time learning algorithms 
that can tolerate nasty noise of {\em constant} rate, i.e.,
we want to achieve error guarantees that are {\em independent of the dimension}.
In recent work,~\cite{DiakonikolasKS18a-nasty} obtained such an algorithm for degree-$d$ PTFs
over the Gaussian distribution (and under well-behaved continuous distributions). 

The robust learning algorithm of~\cite{DiakonikolasKS18a-nasty} for degree-$d$ PTFs 
works in two steps: (1) Start by robustly approximating the degree-$d$ Chow parameters
of our function, and (2) Use the approximate degree-$d$ Chow parameters from Step (1) 
to find a proper hypothesis that is close to the target concept.

Step (2) uses known algorithmic techniques~\cite{TTV:09short, DeDFS14} 
(in particular, Theorem~\ref{thm:alg-chowd}) to efficiently find an accurate 
proper hypothesis with approximately these degree-$d$ Chow parameters.
The correctness of Step (2) leverages the fact that {\em approximations}
to the degree-$d$ Chow parameters information-theoretically approximately determine our function.
While such a structural result is easy to show for continuous well-behaved distributions,
it was very challenging to prove for the uniform distribution on the hypercube. 
Our Theorem~\ref{thm:chow-d-alg} provides the required structural result for degree-$d$ PTFs
over $\bn$, hence allowing the~\cite{DiakonikolasKS18a-nasty} algorithmic approach to go through.
This gives Theorem~\ref{thm:malicious}.

\section{Conclusions and Open Problems} \label{sec:conc}
In this paper, we showed that the degree-$d$ Chow parameters of a degree-$d$ PTF $f$
robustly determine $f$ with respect to $\ell_1$-distance. As a corollary, we obtained a number
of algorithmic and structural applications.

The main qualitative message of our result is that the relation between degree-$d$ Chow
distance and $\ell_1$-distance is independent of the dimension $n$. 
On the other hand, our quantitative bounds can almost certainly be improved. 
The obvious open problem is to prove a nearly tight relation between
the two metrics. More specifically, what is the best possible function $\chowallow(\eps, d)$ 
in the statement of Theorem~\ref{thm:chow-d-struct}? It is known that for $d=1$ the answer
is between $\eps^{\Omega(\log \log (1/\eps))}$ and $\eps^{O(\log^2(1/\eps))}$. For general $d$, standard 
bounds on approximating arbitrary degree-$d$ PTFs by degree-$d$ PTFs with small integer weights imply 
that $\chowallow(\eps, d)$ cannot be larger than $\eps^{\Omega_d(\log(1/\eps)^{d-1})}$. We believe
that obtaining such improved bounds requires new ideas.

\bibliographystyle{alpha}
\bibliography{allrefs}

\end{document}